\documentclass[sigconf]{acmart}

\usepackage{algorithm}
\usepackage[noend]{algpseudocode}

\algnewcommand{\algorithmicforeach}{\textbf{for each}}
\algdef{SE}[FOR]{ForEach}{EndForEach}[1]
  {\algorithmicforeach\ #1\ \algorithmicdo}
  {\algorithmicend\ \algorithmicforeach}

\usepackage{booktabs} 

\usepackage{hyperref}

\usepackage{mkolar_definitions}
\usepackage{graphicx}
\usepackage{subfigure}
\usepackage{url}

\newtheorem{assumption}{Assumption}
\newtheorem{problem}{Problem}

\theoremstyle{definition}

\newcommand{\bigCI}{\perp}
\newcommand{\sainyam}[1]{{#1}}
\newcommand{\nbigCI}{\not\perp}

\newcommand{\Pa}{\texttt{Pa}}

\definecolor{mygreen}{rgb}{0,0.6,0}
\definecolor{mygray}{rgb}{0.5,0.5,0.5}
\definecolor{mymauve}{rgb}{0.58,0,0.82}

\newcommand{\reva}[1]{{#1}}
\newcommand{\revb}[1]{{#1}}
\newcommand{\revc}[1]{{#1}}

\copyrightyear{2022}
\acmYear{2022}
\setcopyright{acmcopyright}\acmConference[SIGMOD '22]{Proceedings of the 2022 International Conference on Management of Data}{June 12--17, 2022}{Philadelphia, PA, USA}
\acmBooktitle{Proceedings of the 2022 International Conference on Management of Data (SIGMOD '22), June 12--17, 2022, Philadelphia, PA, USA}
\acmPrice{15.00}
\acmDOI{10.1145/3514221.3517909}
\acmISBN{978-1-4503-9249-5/22/06}
\begin{document}
\fancyhead{}

\title{Causal Feature Selection for Algorithmic Fairness}

\author{Sainyam Galhotra}
 \affiliation{
   \institution{University of Chicago}
   \country{}
 }
 \email{sainyam@uchicago.edu}

\author{Karthikeyan Shanmugam}
 \affiliation{
   \institution{IBM Research AI}
   \country{}
 }
 \email{karthikeyan.shanmugam2@ibm.com}

\author{Prasanna Sattigeri}
 \affiliation{
   \institution{IBM Research AI}
   \country{}
 }
 \email{psattig@us.ibm.com}

\author{Kush R. Varshney}
 \affiliation{
   \institution{IBM Research AI}
   \country{}
 }
 \email{krvarshn@us.ibm.com}

\begin{abstract}
The use of machine learning (ML) in high-stakes societal decisions has encouraged the consideration of fairness throughout the ML lifecycle.
Although data integration is one of the primary steps to generate high quality training data, most of the fairness literature ignores this stage. In this work, 
we consider fairness in the integration component of data management, aiming to identify features that improve prediction without adding any bias to the dataset. 
We work under the \emph{causal fairness} paradigm~\cite{capuchin}. Without requiring the underlying structural causal model a priori, we propose an approach to identify a sub-collection of features that ensure fairness of the dataset by performing conditional independence tests between different subsets of features.  We use group testing to improve the complexity of the approach. 
We theoretically prove the correctness of the proposed algorithm and show that sub-linear conditional independence tests are sufficient to identify these variables. A detailed empirical evaluation is performed on real-world datasets to demonstrate the efficacy and efficiency of our technique.
\end{abstract}

\ccsdesc[500]{Theory of computation~Machine learning theory}

\keywords{Causal fairness, feature selection, fair machine learning}
\maketitle
\section{Introduction\label{sec:intro}}
Algorithmic fairness is of great societal concern when supervised classification models are used to support allocation decisions in high-stake applications. There have been numerous recent advances in statistically and causally defining group fairness between populations delineated by protected attributes and in the development of algorithms to mitigate unwanted bias \citep{BarocasHN2020}.\footnote{We use the terms \emph{sensitive attribute} and \emph{protected attribute} interchangeably.} Bias mitigation algorithms are often categorized into pre-processing, in-processing, and post-processing approaches. Pre-processing techniques modify the distribution of the training data, in-processing techniques modify the objective function of the training procedure or consider additional constraints in the learning phase, and post-processing techniques modify the output predictions --- all in service of improving fairness metrics while upholding classification accuracy \citep{DalessandroOL2017}.  Table~\ref{tab:techniques} summarizes a representative set of prior bias mitigation algorithms. 
{\small
\begin{table}[]
\begin{center}
 \begin{tabular}{|c|c| c|} 
 \hline
  & Associational & Causal  \\ [0.5ex] 
 \hline
 Pre/Post-processing & \cite{calmon2017optimized,feldman2015certifying,kamiran2012data} & \cite{capuchin,chiappa2019path,jiang2019wasserstein}  \\ 
 \hline
 In-processing &\cite{kamishima2012fairness,zafar2017fairness,calders2010three,celis2019classification,hardt2016equality}&\cite{nabi2018fair,russell2017worlds}\\
 \hline
 Feature Selection &&\\
  Discard biased attributes & - & \textbf{This paper}  \\
 \hline
\end{tabular}
\end{center}
\caption{Different categories of fairness techniques.}
\vspace{-11mm}
    \label{tab:techniques}
\end{table}
}
However, this categorization misses an important stage in the lifecycle of machine learning practice: data collection, engineering and management \citep{SchelterHKS2019,JoG2020}. \citet{HolsteinWDDW2019} report that practitioners ``typically look to their training datasets, not their ML models, as the most important place to intervene to improve fairness in their products''.
Data integration, one of the first components of data management, aims to join together information from different sources that captures rich context and improves predictive ability. With the phenomenal growth of digital data, ML practitioners may procure features from millions of sources spanning data lakes, knowledge graphs, etc~ \citep{miller2018open,galhotraautomated}. They typically generate exhaustive sets of features from all sources and then perform subset selection~\citep{zhang2016materialization,konda2013feature,galhotraautomated}. \revb{Feature selection is a promising direction for fairness in ML as it does not require assumptions about data distribution and is robust to distribution shifts~\cite{diazautomated}, assuming distribution shifts do not change the structural aspects of the causal model.} Some may argue that data integration is a part of pre-processing but we make this distinction as data integration does not involve modification of the data distribution and is considered as the task of a data engineer as opposed to a data modeler.

Filtering methods for feature selection exploit the correlation of features to identify a subset~\citep{hall1999correlation}. However, these techniques are ignorant of sensitive attributes and fairness concerns. For example, consider a dataset with features $F_1$ and $F_2$ such that $F_1$ provides slightly more improvement in accuracy than $F_2$; however, incorporating $F_1$ yields a classifier that reinforces discrimination against protected groups whereas incorporating $F_2$ yields a classifier with similar outcomes for different groups. Feature selection techniques that are not discrimination-aware will prefer $F_1$ to $F_2$, but $F_2$ is a better feature to select from a societal perspective.  

{To overcome the fairness limitations of standard feature selection methods, we study the problem of fair feature selection, specifically in the context of data integration when we are integrating new tables of features with an existing training dataset (PK-FK joins) or source selection or generating new features using transformations~\cite{khurana2016cognito,diazautomated,galhotraautomated}.} \textit{Our goal is to identify a subset of new features\footnote{{Our algorithms do not assume that all features are presented a priori and works in case new features are  added incrementally.}} that can be integrated with the original dataset without worsening its biases against protected groups.} As an additional advantage,  the feature selection paradigm is known to be stable against changes in data distribution as compared to prior techniques that modify the output predictions or the data distribution to mitigate bias~\citep{singh2019fair}. Following the framework of prior fair algorithms~\cite{capuchin,chiappa2018causal,chiappa2019path}, we assume access to protected/sensitive attributes which are used to identify the feature subset that obeys fairness. 
The identification of features that do not induce additional bias is tricky because of relationships between non-protected attributes and protected ones that allow the reconstruction of information in the protected attributes from one or more non-protected ones. For example, zip code can reconstruct race  information~\citep{zip}.

There are two main types of techniques to ensure fairness in data: Associational and Causal (summarized in Table~\ref{tab:techniques}). Associational techniques look for associative relationships between sensitive attributes and the prediction outcome to mitigate unwanted biases. However, these techniques are based on correlation between attributes and fail to capture causal relationships.
There has been a lot of interest in studying causal frameworks~\citep{chiappa2019path, xu2019achieving, KusnerLRS2017, jiang2019wasserstein, chiappa2018causal, KilbertusRPHJS2017,ZhangB2018a,ZhangB2018b,khademi2019algorithmic,khademi2019fairness, russell2017worlds} to achieve fairness.  Due to their ability to distinguish different discrimination mechanisms,  we use \emph{causal fairness}~\citep{capuchin,loftus2018causal} as our fairness framework.  Certain causal approaches assume access to the underlying causal structure, which is unrealistic in practice~\cite{chiappa2019path, xu2019achieving, KusnerLRS2017, jiang2019wasserstein, chiappa2018causal, KilbertusRPHJS2017,ZhangB2018a}. \emph{Importantly, we do not make the assumption that we are given the causal graph (formally, the structure of the causal bayesian network that generates the data) a priori.}

We propose an algorithm \texttt{SeqSel} to identify all new features that when added to the original dataset still ensure causal fairness. Our algorithm takes as input a dataset $D$ comprising an outcome variable, sensitive features, admissible features, and a collection of features that are neither admissible nor sensitive. A feature is considered \emph{admissible} if the protected variables are allowed to affect the outcome through it. \reva{For example, consider a credit card application system that contains gender and race as sensitive attributes, expected monthly usage as an admissible attribute (it may have a sensitive attribute as one of its parent but it is permissible for the sensitive attribute to influence the outcome through this variable), and age and education level as variables which are neither sensitive nor admissible.} A set of features  $\mathbf{X}$ is considered to ensure causal fairness  if after adding these features one could increase accuracy of a subsequently trained classifier on this new dataset without worrying about causal fairness metrics, i.e. in effect the subset of features when added does not introduce any tradeoff between fairness and accuracy and they are safe to subsequent attempts at building a purely predictive classifier. Our approach operates in two phases focused towards performing conditional independence tests  with respect to the sensitive attributes and the target variable. These tests help identify variables that (1) do not capture information about sensitive attributes, or (2) ensure fairness even if they capture some information about sensitive attributes.  
We theoretically prove that both types of these variables ensure causal fairness and analyze the conditions to identify all such variables.

The na{\"\i}ve \texttt{SeqSel} algorithm performs a number of conditional independence tests that grows linearly in the number of features in the dataset. One of the major shortcomings of extant conditional independence testing methods is that they generate spurious correlations between variables if too many tests are performed~\citep{strobl2019approximate}.  To overcome this limitation and reduce the chances of getting spurious results, we propose a more efficient algorithm, \texttt{GrpSel}, that uses graphoid axioms to show that \emph{group testing} can reduce the number of tests to the logarithm of the number of features and additionally  improves the overall efficiency of the pipeline. 

  Our primary contributions are:
  \begin{itemize}
      \item We formalize the problem of fairness in data integration and feature selection setting using causal fairness.
      \item We provide an algorithm that performs conditional independence tests to identify the variables that do not worsen the fairness of the dataset.
      \item We prove theoretical guarantees that the variables identified by our algorithm ensure fairness and identify a closed form expression for variables that cannot be added.
      \item We propose an improved algorithm that leverages ideas of \emph{group testing} to reduce the chances of getting spurious correlations and has sub-linear complexity.
      \item We show empirical benefits of our techniques on synthetic and real-world datasets.
  \end{itemize}
The paper represents a principled use to address an important problem that has not been addressed before: fair data integration.

\section{Preliminaries\label{sec:prelim}}
In this section, we review the background on algorithmic fairness and models of causality.

We denote variables (also known as dataset attributes or features) by uppercase letters like $X,S,A$, corresponding values in lower case like $x,s,a$\reva{, and sets of attributes or values in bold ($\mathbf{X}$ or $\mathbf{x}$).  }

\subsection{Algorithmic Fairness\label{sec:fairdef}}
\reva{The area of algorithmic fairness aims to ensure unbiased output for different sub-groups identified by specific set of attributes (also known as protected or sensitive attributes). For example, a loan prediction software should not discriminate against female applicants (gender is the protected attribute). }
The literature on algorithmic fairness considers  a set of protected attributes $\mathbf{S}=\{S_1,\ldots, S_{|\mathbf{S}|}\}$, a target variable $Y$ and a prediction algorithm $f:\mathbf{V}\rightarrow Y$ where $\mathbf{V}$ denotes the set of input attributes and the output of $f$ is called the prediction output or an outcome. \reva{Typically, ML tasks train a classifier on a dataset $D$ (comprising of attributes $\mathbf{V}$ and target $Y$) which is assumed to be distributed according to a distribution $\Pr$.}  In order to measure the fairness of $f$ \reva{with respect to $\mathbf{S}$}, two different types of metrics have been studied: Associational and Causal. 

\noindent \textbf{Associational fairness} methods capture statistical variabilities in the behavior of the prediction algorithm for different groups of individuals. For example, equalized odds requires that
the false positive and true positive rate of different sub-groups identified by the sensitive attributes is the same. Other associational fairness measures include Demographic parity, conditional statistical parity, and predictive parity~\cite{kamishima2012fairness,zafar2017fairness,calders2010three,celis2019classification,hardt2016equality,calmon2017optimized}. 
Even though associational methods of quantifying fairness are very popular, all these methods fail to distinguish
between causal influence and spurious
correlations between different input attributes
of the prediction algorithm~\cite{capuchin}. To this end, recent methods have proposed to capture the causal dependence of the outcome on the protected attribute. Before describing these methods, we present a background on causal graphs.

\subsection{Causal DAGs\label{sec:causal}}

\noindent \textbf{Probabilistic Causal DAG.} A causal DAG over a set of variables $\mathbf{V}$ is a directed acyclic graph $G$ that captures \revc{functional dependencies} between these variables. A variable $X_1$ is considered to cause $X_2$ iff $X_1\rightarrow X_2$ in the causal DAG $G$.  Each variable in the causal graph $G$ is functionally determined by its parents and some unobserved exogenous variables. The causal graph is used as a compact representation to denote the dependence between different variables. 
\reva{Two variables $X$ and $Y$ are independent when conditioned on $Z$ if $\Pr(Y=y|X=x,Z=z)=\Pr(Y=y|X=x)$ and is denoted by $X\bigCI Y |_{\Pr} Z$. }
To test this condition, \reva{we consider a conditional independence (CI) test~\cite{strobl2019approximate}} that returns if $X$ and $Y$ are independent conditioned on $Z$. An orthogonal line of work has studied different techniques to efficiently test this condition~\cite{strobl2019approximate}. \revc{The joint probability distribution of a set of variables $\mathbf{V}$ can be decomposed similar to that of bayesian networks, 
{\small
\begin{align}
    \Pr(\mathbf{V})&=\prod_{X\in \mathbf{V}} \Pr(X| \Pa(X)),
\end{align}}
where $\Pa(X)$ denotes the set of parents of $X$ in the graph $G$.}

\noindent \reva{ \textbf{d-separation and Faithfulness} One of the common questions that are answered using causal DAGs is whether $\mathbf{X}\bigCI \mathbf{Y} | \mathbf{Z}$, i.e. a set of variables $\mathbf{X}$ is independent of $\mathbf{Y}$, conditioned on $\mathbf{Z}$. d-separation between three sets of variables $\mathbf{X},\mathbf{Y},\mathbf{Z}$, denoted by $\mathbf{X}\bigCI \mathbf{Y} |_d \mathbf{Z}$, is a sufficient graphical criterion that syntactically captures observed conditional independencies. $\mathbf{X}$ and $\mathbf{Y}$ are said to be d-separated given $\mathbf{Z}$, if all paths between  $\mathbf{X}$ and $\mathbf{Y}$ are blocked by $\mathbf{Z}$ (Please refer to the full version~\cite{fullversion} for a formal definition of blocking and d-separation). Probability distribution of a dataset $D$ is said to be markov compatible~\cite{pearl2009causality} if d-separation implies CI with respect to the probability  distribution $\Pr$. If the converse also holds ($X\bigCI Y |_{\Pr}Z \implies X\bigCI Y |_{d}Z$ ), the probability distribution $\Pr$ is considered faithful to the causal graph $G$~\cite{peters2017elements}. We assume throughout this work that $\Pr$ is markov compatible and faithful to $G$. As CI and d-separation are equivalent under these assumptions, we ignore the sub-script $\Pr$ or $d$ in subsequent discussions. Faithfulness is a standard assumption in causal inference, which ensures that all CI observed in the dataset correspond to d-separations in the corresponding causal graph~\cite{peters2017elements,lauritzen2018unifying,pearl2009causality}.   Graphoid axioms~\cite{sadeghi2017faithfulness,lauritzen2018unifying,pearl2009causality} are the popular set of properties that are used to infer conditional independence. We list two axioms that are relevant for this study.
\begin{lemma}[Theorem 1~\cite{lauritzen2018unifying}]
Consider a dataset $D$ with a causal graph $G$, where the data distribution $\Pr$ is faithful to the graph $G$. 
\begin{enumerate}
    \item \textbf{Decomposition axiom}: If $A\bot B,C|Z$, then $A\bot B | Z$ and $A\bot C | Z $ 
    \item \textbf{Composition axiom}: If $A\bot B | Z$ and $A\bot C | Z $, then $A\bot B,C | Z$
\end{enumerate}
\end{lemma}
\begin{proof}
We use the notion of d-separation to prove these results. 

\noindent \emph{Decomposition axiom:} If $A\bot B,C|Z$, then all paths from $A$ to any of $B$ or $C$ are blocked given $Z$. Therefore, any path from $A$ to $B\subseteq B\cup C$ is also blocked given $Z$. Therefore, $A\bot B|Z$. Symmetrically, the same argument proves that $A\bot C|Z$. Therefore, $A\bot B,C |Z$. 
\end{proof}
}

\noindent \textbf{\texttt{do}-operator.}  \reva{Pearl~\cite{pearl2009causality} defined intervention as a modification of the state of attributes to a specific value and observe its effect.} An intervention on an attribute $X\leftarrow x$ is equivalent to assigning a value $x$ to the variable $X$ \revc{in a modified causal graph $G'$, where $G'$ is same as $G$ except that all incoming edges of $X$ have been removed.}
According to Pearl~\cite{pearl2009causality}, do-operator is equivalent to the graphical interpretation of an intervention. An intervention $\texttt{do}(X)=x$ is equivalent to conditioning $X=x$ if $X$ has no ancestors in $G$. 

\subsection{\reva{Causal Fairness\label{sec:causalfair}}}
\reva{There has been a lot of recent interest in studying the causal impact of protected attributes on the prediction variable. } 
\reva{Causal measures capture the causal dependence of the prediction variable on the sensitive attributes and aim to minimize such effects at different population levels. 
}

\noindent \reva{\textbf{Admissible Attributes.} In an ideal setting, the prediction attribute and the protected attributes should be d-separated in the causal graph whenever we intervene on the protected attributes. However, it is a very restrictive and impractical requirement~\cite{capuchin}. To improve the usefulness of this definition, a subset of the attributes are labelled admissible, through which protected attribute is allowed to impact the prediction attribute. For example, applicant's choice of loan type or loan duration in a banking application.
The set of admissible attributes also help to understand the impact of different attributes on the prediction accuracy and fairness. The specification of attributes as admissible is application-dependent and are considered as an input to the problem.
}

One of the recent causal fairness definitions, interventional fairness~\cite{capuchin} is the strongest notion of fairness that is testable over the input dataset and correctly captures group level fairness. It assumes that the input attributes $\mathbf{V}$ consist of  admissible attributes $\mathbf{A}$, through which the sensitive attributes are allowed to influence the prediction output.  \revc{The fairness definition in~\cite{capuchin} was designed to study datasets and focused on the target attribute $Y$. We extend this definition to analyze fairness of ML classifiers by analyzing the effect of sensitive attributes on $Y'$, the prediction output.}

\begin{definition}[Causal Fairness]
\label{def:causal-fair1}
For a given set of admissible variables, $\mathbf{A}$, a classifier is considered fair if for any collection of values $a$ of $\mathbf{A}$ and output $Y'$, the following holds: $Pr(Y'=y|\texttt{do}(\mathbf{S})=\mathbf{s},\texttt{do}(\mathbf{A}=\mathbf{a}) ) = Pr(Y'=y|\texttt{do}(\mathbf{S})=\mathbf{s}',\texttt{do}(\mathbf{A}=\mathbf{a}) )$ for all values of $\mathbf{A}$, $\mathbf{S}$ and ${Y'}$.
\end{definition}
\begin{example}
Consider a loan prediction software~\cite{german} that considers demographic attributes along with credit information and loan preferences. Among input attributes, race and gender are considered protected and loan preferences like loan type and duration are generally considered admissible because any bias due to sensitive attributes is allowed to affect the outcome only if it is through individual's preferences. Other attributes like age, zip-code, income, education, etc are considered neither admissible nor inadmissible. 

\end{example}
In this dataset, some features like zip-code have been identified as proxy features which are causally dependent on race. Using any of these proxy features for classifier training can inject bias into the system. 
\reva{According to Definition~\ref{def:causal-fair1}, the protected attributes $\mathbf{S}$ are independent of $Y'$ conditioned on $\mathbf{A}$ in the intervened graph (incoming edges of $\mathbf{S}$ and $\mathbf{A}$ are removed), say $G'$. In other words, $\mathbf{S}$ and $Y'$ are d-separated conditioned on $\mathbf{A}$ in $G'$. For more insights about the definition of causal fairness, we refer the reader to~\cite{capuchin}. Recent work has also studied causal fairness in settings where the protected attribute is unobserved~\cite{galhotra2021interventional}.}

\noindent \revc{\textbf{Testing causal fairness.} Causal fairness is an interventional definition that is represented using \texttt{do} operators. A straightforward way to test this definition is to leverage a fully specified causal graph (graph structure and equations) to estimate the post-intervention probability values. However, fully specified causal graphs are not available in practice and this definition can not be tested as is. Instead, we present a sufficient condition to test for causal fairness.
\begin{lemma}
If conditional-mutual information between the classifier output $Y'$ and protected attributes $\mathbf{S}$ is zero when conditioned on the admissible set $\mathbf{A}$, i.e., $I(Y',\mathbf{S} | A=\mathbf{a}) =0$ then $Y'$ is causally fair.\label{lem:cmi}
\end{lemma}

}

\begin{figure*}
\vspace{-5pt}
\subfigure[$X_2$  is a biased variable\label{fig:diag1}]{\includegraphics[width=0.26\textwidth]{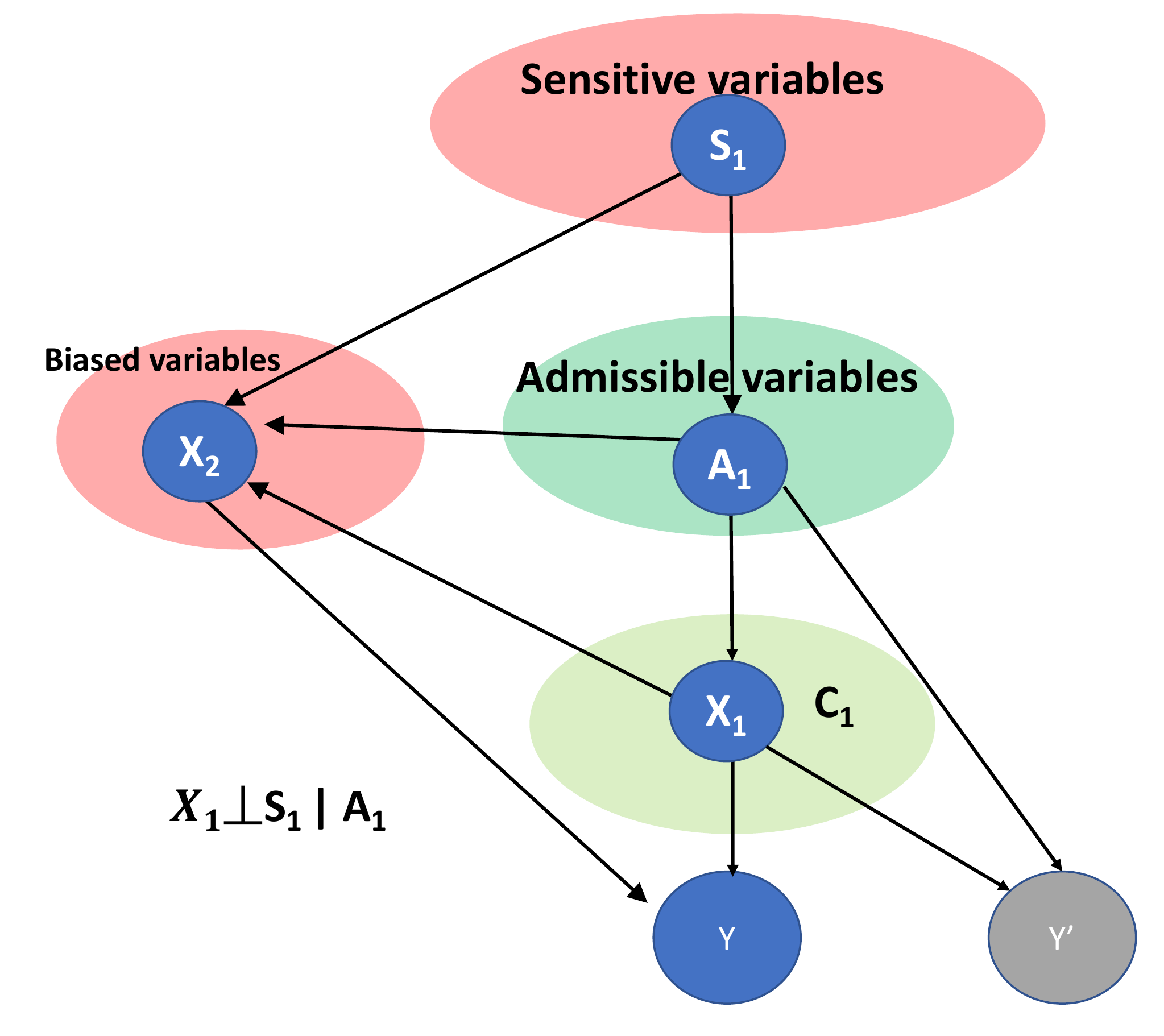}}
\subfigure[$X_1,X_2, X_3$ ensure fairness\label{fig:diag2}]{\includegraphics[width=0.28\textwidth]{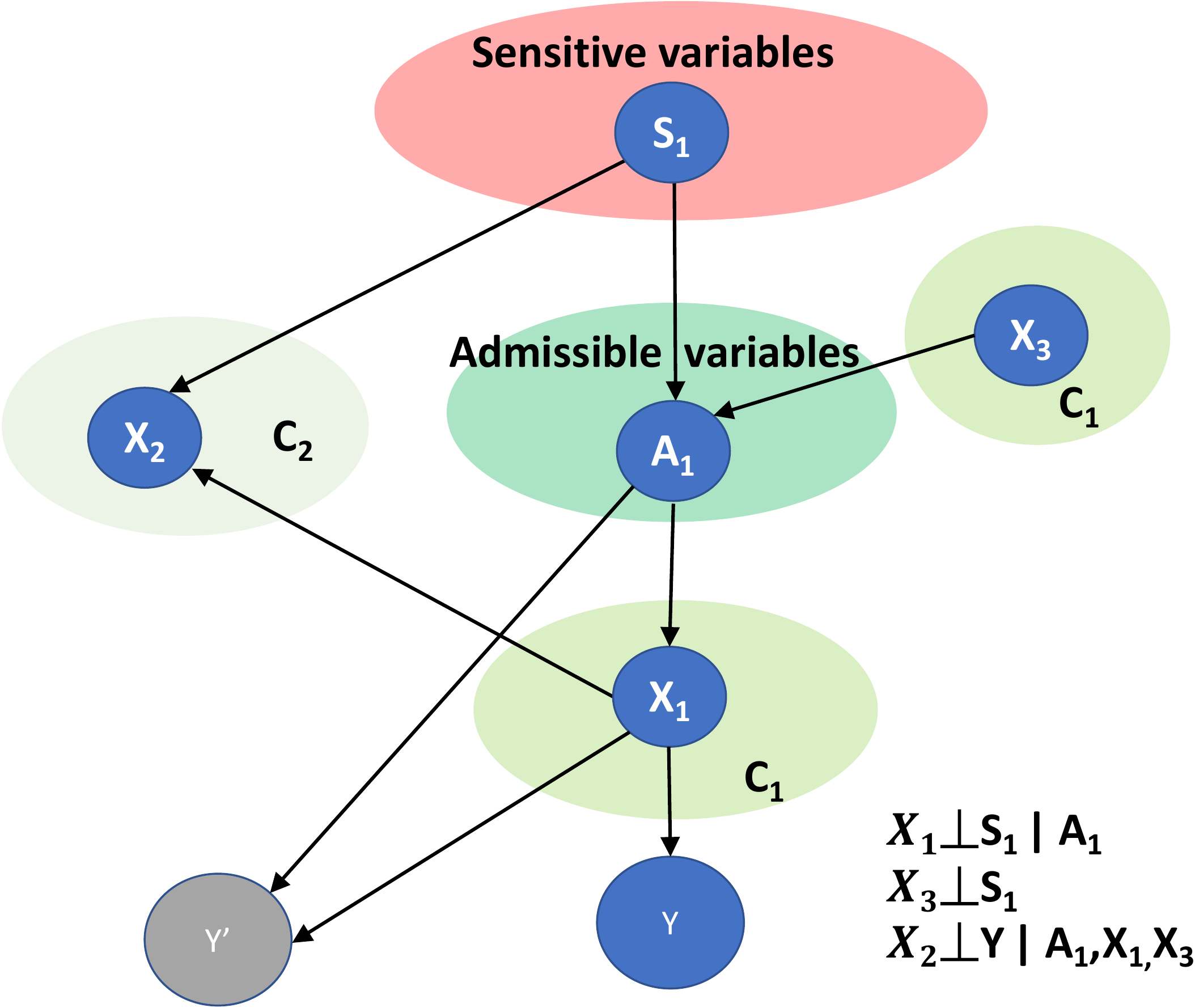}}
\subfigure[$X_1,X_2, X_3$ ensure fairness\label{fig:diag3}]{\includegraphics[width=0.29\textwidth]{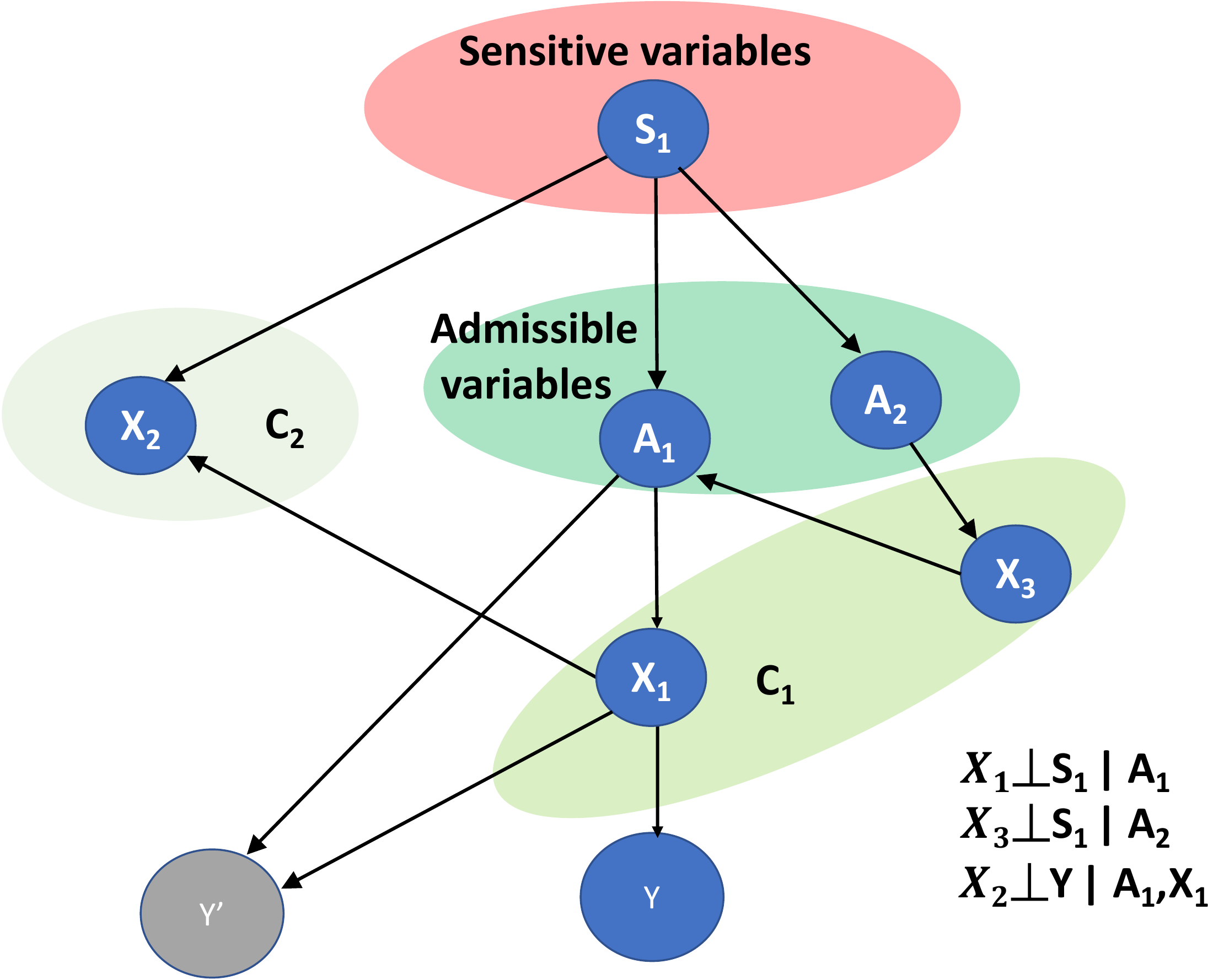}}
\vspace{-5mm}
\caption{Example causal graphs that demonstrate different types of variables. 
\vspace{-5mm}
\label{fig:fairjoin}}
\end{figure*}
\vspace{-3mm}
\section{Problem Statement\label{sec:problem}}
In this section, we define the problem of feature selection to ensure interventional fairness and provide high level intuition of the involved challenges.

Consider a dataset $D$ comprising of a disjoint set of two types of features (i)  Sensitive $\mathbf{S}=\{S_1,\ldots, S_{|\mathbf{S}|}\}$ and (ii) Admissible  $\mathbf{A}=\{A_1,\ldots, A_{|\mathbf{A}|}\}$ along with a target variable $Y$.  Let $\mathbf{X}= \{X_1,\ldots, X_n\}$ denote the collection of $n$ features that are neither admissible nor sensitive and can be added to $D$ by performing a join between the input dataset and different datasets from different sources or by feature transformation over a subset of the features.  Let $\mathbf{V} = \mathbf{A}\cup \mathbf{S} \cup \mathbf{X} \cup Y$ denote the exhaustive list of available variables and $Y'$ denote the learnt target variable which has been trained over a subset $\mathbf{T}\subseteq\mathbf{V}$. 
 Now, we present the definition of causally fair features that can be added to the original dataset.

\begin{definition}[Causally Fair Features]
\label{def:causal-fair}
For a given set of admissible variables, $\mathbf{A}$, we say a collection of features $D=\mathbf{A}\cup \mathbf{T}$ is causally fair if the bayes optimal predictor $Y'$, trained on $D$ satisfies causal fairness with respect to sensitive attributes $\mathbf{S}$.
\end{definition}

The goal is to identify the largest subset $\mathbf{T}\subseteq\mathbf{V}$ such that the  variable $Y'$, trained using these variables is fair. 

\begin{problem}\label{prob:fair}
Given a dataset $D=\{\mathbf{A},\mathbf{S},Y\}$ and a collection of  variables $\mathbf{X}$, identify the largest subset $\mathbf{T}\subseteq \mathbf{X}$ such that the  features  $D'=\mathbf{A}\cup  \mathbf{T}$ is causally-fair.
\end{problem}

The goal of our problem is to identify all features
that can be considered for training a classifier without  worsening the fairness of the dataset $D$.  Note that $D$ contains only features $\mathbf{S} \cup  \mathbf{A}$ to begin with, so there is no fairness violation as sensitive attributes are allowed to influence $Y'$ through $\mathbf{A}$ and $\mathbf{S}$ are not used for  training. We make the following assumptions about the causal graph:
\begin{assumption}[Faithfulness assumption]
 The causal graph $G$ on $\mathbf{V}$ is faithful to the observational distribution on $\mathbf{V}$.\label{faithfulness}
\end{assumption}
This assumption implies that if two variables $A$ and $B$ are connected in the causal graph, the data cannot result in any spurious conditional independency of the form $(A\perp B |\mathbf{C})$ for any subset $\mathbf{C} \subset \mathbf{V} \setminus \{A,B\} $.
\reva{Faithfulness assumption is one of the most common assumptions in causality and fairness literature~\cite{chiappa2019path, xu2019achieving, KusnerLRS2017, jiang2019wasserstein, chiappa2018causal, KilbertusRPHJS2017,ZhangB2018a,ZhangB2018b,khademi2019algorithmic,khademi2019fairness, russell2017worlds,capuchin}, which is crucial to model the input dataset. }

\noindent \textbf{Classifier Training.} A new variable $Y'$ (prediction variable) is generated by learning a predictor over the selected subset of features ($ \mathbf{A} \cup \mathbf{T}$), and this predictor is the \sainyam{Bayes optimal classifier with}  $Pr[Y'|\mathbf{A} \cup \mathbf{T}]$ derived from the observational distribution $P(\mathbf{V})$. \reva{It is equivalent to adding $Y'$ as a new node in the causal graph which is a children of all features that impact the classifier output. } 
We make Assumption~\ref{yprime-assumption} to ensure that one would apply the same \sainyam{Bayes} optimal predictor that has been learnt from observational data to all datasets irrespective of the intervention. \sainyam{This assumption is crucial to decouple fairness of feature selection from the training procedure and to theoretically analyze the quality of bias removal in feature selection.  Training the classifier by performing  feature engineering over the identified features satisfies this assumption.}
\begin{assumption}\label{yprime-assumption}
For evaluating the fairness criterion in Definition \ref{def:causal-fair} using hypothetical interventional distributions, we assume that the mechanism generating $Y'$ is the same as $P[Y'|\mathbf{A} \cup \mathbf{T}]$ where $P(\cdot)$ is the observational distribution.
\end{assumption}

\noindent\textbf{Problem intuition:}
According to the definition of causal fairness, the output distribution of the prediction algorithm should not change when the value of sensitive variables  is changed \sainyam{whenever we intervene on $\mathbf{A}$. According to \texttt{do}-calculus, intervention on ($\mathbf{A}$) is equivalent to removal of its incoming edges and conditioning on $\mathbf{A}$. } If all paths from the sensitive variables  to the learnt target $Y'$ that go through the variables considered by $f$ are blocked \sainyam{after an intervention on the admissible variables}, then the  features considered by $f$ are causally-fair. 
\reva{We first show that the maximal set of features that ensure causal fairness is unique.}
\vspace{-2mm}
\reva{\begin{lemma}
Consider two different set of attributes $\mathbf{X_1}$ and $\mathbf{X_2}$ such that $\mathbf{X_1}\neq \mathbf{X_2}$. If a classifier trained on $\mathbf{X_1}$ and $\mathbf{X_2}$ separately is causally fair, then a classifier trained on $\mathbf{X_1}\cup \mathbf{X_2}$ is also causally fair.
\label{lem:union}
\end{lemma}
\vspace{-5mm}
\begin{proof}
Let $Y_1'$ and $Y_2'$ denote the output variable of the classifier trained on $\mathbf{X_1}$ and $\mathbf{X_2}$. Let $G'$ denote a modified causal graph where incoming edges of $\mathbf{S}$ and $\mathbf{A}$ are removed.
According to the definition of causal fairness, all paths from the sensitive atrributes to $Y_1'$ are blocked in $G'$, i.e. $ \mathbf{S}\bot Y_1' |_{G'} \mathbf{A}$. Since, $Y_1'$ is a child of attributes in $\mathbf{X_1}$, all paths from $\mathbf{S}$ to the parents of $Y_1'$ are blocked, i.e., $\Pa(Y_1')\bot \mathbf{S} |_{G'} \mathbf{A}$. We get the same condition for $\mathbf{X_2}$.
Let $Y' = f(\mathbf{X_1\cup X_2})$. We first simplify the LHS of causal fairness definition as follows.
{\footnotesize
\begin{align*}
   & \Pr(Y'=y | \texttt{do}(\mathbf{S})=\mathbf{s}, \texttt{do}(\mathbf{A})=\mathbf{a}) \\
    &=\sum_{\Pa(Y')=\mathbf{c}} (\Pr(Y'=y | \Pa(Y')=\mathbf{c}, \texttt{do}(\mathbf{S})=\mathbf{s}, \texttt{do}(\mathbf{A})=\mathbf{a})\times  \\&~~~~~~~~~~~~~~ \Pr(\Pa(Y')=\mathbf{c}|\texttt{do}(\mathbf{S})=\mathbf{s}, \texttt{do}(\mathbf{A})=\mathbf{a}))\\
    &=\sum_{\Pa(Y')=\mathbf{c}} \Pr(Y'=y | \Pa(Y')=\mathbf{c}) \Pr(\Pa(Y')=\mathbf{c}|\texttt{do}(\mathbf{S})=\mathbf{s}, \texttt{do}(\mathbf{A})=\mathbf{a})\\
    &=\sum_{\Pa(Y')=\mathbf{c}} \Pr(Y'=y | \Pa(Y')=\mathbf{c}) \texttt{Pr}_{G'}(\Pa(Y')=\mathbf{c}|\mathbf{S}=\mathbf{s}, \mathbf{A}=\mathbf{a})
\end{align*}
}
Since, $Y'$ is trained over $\mathbf{X_1}$ and $\mathbf{X_2}$,  $\Pa(Y') \subseteq \mathbf{X_1}\cup \mathbf{X_2}$. Therefore,  $Pa(Y')\bot_{G'} \mathbf{S} | \mathbf{A}$, implying $\texttt{Pr}_{G'}(\Pa(Y')=\mathbf{c}|\mathbf{S}=\mathbf{s}, \mathbf{A}=\mathbf{a})=\texttt{Pr}_{G'}(\Pa(Y')=\mathbf{c}|\mathbf{A}=\mathbf{a})$. Following the same simplification on RHS of Definition~\ref{def:causal-fair}, we get that $\mathbf{X_1\cup X_2}$ are causally fair.
\end{proof}}

\vspace{-2mm}
\reva{Using Lemma~\ref{lem:union}, we prove that problem~\ref{prob:fair} has a unique solution.
\begin{lemma}\label{lem:maximal}
Problem~\ref{prob:fair} has a unique solution $\mathbf{T}^*$.
\end{lemma}
\vspace{-5mm}
\begin{proof}
Suppose, Problem~\ref{prob:fair} does not have a unique solution. Let $\mathbf{T}_1$ and $\mathbf{T}_2$ be two different maximal sets of features that ensure causal fairness. Using Lemma~\ref{lem:union}, $\mathbf{T}_1\cup \mathbf{T}_2$ also ensure causal fairness. Since $\mathbf{T}_1\neq \mathbf{T}_2$, $ |\mathbf{T}_1\cup \mathbf{T}_2| > |\mathbf{T}_1|, |\mathbf{T}_2|$. This is a contradiction, as $\mathbf{T}_1$ and $\mathbf{T}_2$ are maximal sets. Therefore, the assumption that Problem~\ref{prob:fair} does not have a unique solution is wrong.
\end{proof}}
\vspace{-3mm}
\section{Solution Approach\label{sec:solution}}
\revb{In this section, we first present key properties using an example and generalize them to discuss our algorithm, \texttt{SeqSel}. Section~\ref{sec:theory} analyzes the different steps of Algorithm~\ref{alg:fairjoin} to guarantee causal fairness of identified features and Theorem~\ref{thm:main} presents a close-form expression to identify maximal set of causally-fair features.}
\subsection{Algorithm}
One na{\"\i}ve solution to ensure fairness is to consider only the admissible variables $\mathbf{A}$ for prediction and not add any other feature to the dataset $D$. This would satisfy the fairness condition but achieve poor prediction performance as there may be a variable $X\in \mathbf{X}$ that is highly correlated with the target variable $Y$. Another extreme solution is to consider all the variables of $\mathbf{X}$ for prediction. This approach would yield high predictive performance but can have arbitrarily poor fairness. 
We propose \texttt{SeqSel} (Algorithm~\ref{alg:fairjoin}) which considers the collection of variables $\mathbf{A}$, $\mathbf{S}$ and $\mathbf{X}$ to identify the largest subset of $\mathbf{X}$ which when considered along with $\mathbf{A}$ ensure causal fairness of the learnt variable $Y'$. \sainyam{\texttt{SeqSel} algorithm performs CI tests over the observed data  without explicit knowledge of the underlying causal graph. We use causal graphs only to illustrate the intuition behind the different components of our algorithm.}

 \sainyam{Figure~\ref{fig:fairjoin} presents different   example causal graphs, to understand the solution approach and identify CI tests that can be performed without inferring the complete causal graph.  These graphs contain sensitive variables $\mathbf{S}$, admissible variables $\mathbf{A}$, target variable $Y$ along with other subsidiary variables $X_i$'s.}

\begin{enumerate}
\item \sainyam{In all three figures, variables like $X_1$ have unblocked paths from $\mathbf{S}$ to  $X_1$ but all these paths are blocked by the admissible set. Therefore, these variables do not capture any new information about the protected variables. \reva{In general, such variables can be identified by checking if $X_1$ is conditionally independent of $\mathbf{S}$ given $\mathbf{A}$, i.e. $(X_1\bigCI \mathbf{S} | \mathbf{A})$.} }
\item \sainyam{Variables like $X_3$ in Figure~\ref{fig:diag2} are independent of the sensitive attributes and  can be identified easily by performing CI test between variable $X$ and $\mathbf{S}$.}

\item \sainyam{Variable like $X_3$ in Figure~\ref{fig:diag3} is not independent of $S_1$ but is independent of $S_1$ given $A_2$. $X_3$ ensures causal fairness and can be identified by testing $X_3\bigCI \mathbf{S} | A_2$.  } 

\item \sainyam{$X_2$ in Figure~\ref{fig:diag2} and \ref{fig:diag3} is not independent of $S_1$ even with an intervention on $\mathbf{A}$ and captures sensitive information. However, $X_2$ is independent of $Y$ given $\mathbf{A}$.} 
\end{enumerate}
 {\small
\begin{minipage}[t]{0.47\textwidth}
\vspace{-5mm}
\begin{algorithm}[H]
   \caption{\texttt{SeqSel} }
   \label{alg:fairjoin}
\begin{algorithmic}[1]
   \State {\bfseries Input:} Variables $\mathbf{A}, \mathbf{S}, \mathbf{X}, Y$
   \State $\mathbf{C}_1\leftarrow \phi$
   \For{$X\in \mathbf{X}$}\algorithmiccomment{First phase}
   \If{$\exists A\subseteq \mathbf{A}\text{ such that}(X\bigCI\mathbf{S} | {A}) $} \algorithmiccomment{CI test condition}
   \State $\mathbf{C}_1\leftarrow \mathbf{C}_1 \cup \{X\}$
   \EndIf
   \EndFor
    \State $\mathbf{C}_2\leftarrow \phi$\algorithmiccomment{Second phase}
   \State $\mathbf{X}\leftarrow \mathbf{X}\setminus \mathbf{C}_1 $
   \For{$X\in \mathbf{X}$}
   \If{$(X\bigCI Y | \mathbf{A}\cup \mathbf{C}_1)$}
   \State $\mathbf{C}_2\leftarrow \mathbf{C}_2 \cup \{X\}$
   \EndIf
   \EndFor\\
   \Return  $\mathbf{C}_1\cup \mathbf{C}_2$
   \end{algorithmic}
\end{algorithm}
\end{minipage}
}
 
\reva{The different types of variables considered in points 1-3 above do not capture any sensitive information after intervening on $\mathbf{A}$ or any subset of $\mathbf{A}$.} We denote these variables by $\mathbf{C}_1$, identified by testing CI of $X$ with $\mathbf{S}$ given any subset of $\mathbf{A}$. Therefore, all paths from $\mathbf{S}\rightarrow X \rightarrow Y$ are blocked for all these variables. The variables that capture sensitive information but are independent of $Y$ given all the selected features $\mathbf{C}_1\cup \mathbf{A}$ also do not impact the bayes-optimal classifier. This shows that all the variables discussed above ensure causal fairness. Any variable that is not independent of $\mathbf{S}$ and $Y$ even after intervening on $\mathbf{A}$ is biased and is not safe to be added. $X_2$ in Figure~\ref{fig:diag1} is one such example. 
\reva{Consider a variation in Figure~\ref{fig:diag2} by adding an edge $X_3\rightarrow X_1$. Even then $X_1$ is a valid feature to ensure causal fairness. However, $X_1\nbigCI S_1 | A_1$ and therefore, the above mentioend CI conditions do not capture such variables. Specifically, if a variable $X$ has a blocked path from $S$ which forms a collider at the admissible attribute $A$, then above mentioned CI tests do not capture $X$ in the set of fair features. We discuss this condition more formally in Theorem~\ref{thm:main}.  }
\begin{remark}
In Figure~\ref{fig:diag1}, $X_2\nrightarrow X_1$ because there does not exist any path from $\mathbf{S}$ to $X_1$ which is unblocked given $\mathbf{A}$.
\end{remark}

\begin{remark}
If $\mathbf{C}_2$ is conditionally independent of $Y$ given $\mathbf{A},\ \mathbf{C}_1$, it may not contribute towards the predictive power of the Bayes optimal classifier trained on these variables. However, for most practical purposes the classifier trained can leverage $\mathbf{C}_2$ for better prediction.
\end{remark}
 Algorithm~\ref{alg:fairjoin} captures these intuitions to perform CI tests in two phases. The first phase (lines 3-5) identifies all variables that do not get affected by sensitive attributes, in the presence of admissible attributes $\mathbf{A}$ or any subset of $\mathbf{A}$. All these variables do not capture any extra information about sensitive attributes and are safe to be added to the dataset $D$. The rest of the variables, $\mathbf{X}\setminus \mathbf{C}_1$,  capture information about sensitive attributes which can worsen fairness of the dataset.  The second phase (lines 6-10) identifies the subset such that the target variable is not affected by their sensitive information in the presence of admissible attributes. We call this algorithm \texttt{SeqSel} as it sequentially performs CI tests to select features.

\vspace{-2mm}
\subsection{Theoretical Analysis\label{sec:theory}}
In this section, we show that the variables identified by \texttt{SeqSel} ensure causal fairness. 
We consider the original causal graph $G$ along with a new variable $Y'$ that refers to the prediction variable trained using the variables $\mathbf{A}$ along with the variables returned by Algorithm~\ref{alg:fairjoin}.  We first show that the variables $\mathbf{C}_1$ and $\mathbf{C}_2$ identified by Algorithm~\ref{alg:fairjoin} maintain causal fairness.  For this analysis, we assume that the target variable $Y$ does not have a child. 

\begin{lemma}
Consider a dataset $D$ with admissible variables $\mathbf{A}$ and sensitive $\mathbf{S}$ and  a collection of variables $\mathbf{C}_1$. If $\exists A\subseteq \mathbf{A}$ such that $(\mathbf{C}_1\bigCI \mathbf{S}|{A})$ then $\mathbf{A}\cup \mathbf{C}_1$ is causally fair.\label{lem:c1}
\end{lemma}
\begin{proof}
Given $(\mathbf{C}_1\bigCI \mathbf{S}|{A})$ for some $A\subseteq$, the variable $X$ does not capture any information about the sensitive variables. Hence all paths from $\mathbf{S}$ to the target $Y$ that pass through $X$ are blocked. Mathematically, we consider a causal graph along with $Y'$ and evaluate the distribution under the intervention of  $\mathbf{A}$ and $\mathbf{S}$ as follows.
{\footnotesize
\begin{eqnarray*}
&&Pr[Y'|do(\mathbf{S}),do(\mathbf{A})]= \sum_{\mathbf{C}_1}Pr[Y'|\mathbf{C}_1,do(\mathbf{S}),do(\mathbf{A})] Pr[\mathbf{C}_1|do(\mathbf{S}),do(\mathbf{A})]\\
&&\text{Using Lemma~9 from the full version~\cite{fullversion}}\\
&=& \sum_{\mathbf{C}_1}Pr[Y'|\mathbf{C}_1,do(\mathbf{S}),do(\mathbf{A})] Pr[\mathbf{C}_1|do(\mathbf{A})]\\
&&\text{Using Lemma~10 from the full version}\\
&=& \sum_{\mathbf{C}_1}Pr[Y'|\mathbf{C}_1,do(\mathbf{A})] Pr[\mathbf{C}_1|do(\mathbf{A})]= Pr[Y'|do(A)]
\end{eqnarray*}
}
This shows that any intervention on $\mathbf{S}$ does not affect the variable $Y'$, thereby ensuring causal fairness of the considered features.
\end{proof}
The following lemma justifies the addition of $\mathbf{C}_2$ to the dataset $D$ without affecting its causal fairness.
\begin{lemma}
Consider a dataset $D$ with admissible variables $\mathbf{A}$ and sensitive $\mathbf{S}$, a set of variables $\mathbf{C}_1$ satisfying  $(\mathbf{C}_1\bigCI \mathbf{S}|\mathbf{A})$ and  a collection of variables $\mathbf{C}_2$ with $(\mathbf{C}_2\nbigCI \mathbf{S}|\mathbf{A})$, if $(\mathbf{C}_2\bigCI Y|\mathbf{A},\mathbf{C}_1)$ then $\mathbf{A}\cup \mathbf{C}_2\cup \mathbf{C}_1$ is causally fair.\label{lem:c2}
\end{lemma}
\begin{proof}
We simplify the causal fairness condition as follows:
{\footnotesize
\begin{eqnarray*}
&&Pr[Y'|do(\mathbf{S}),do(\mathbf{A})]\\
&&= \sum_{\mathbf{C}_1,\mathbf{C}_2}\Big(Pr[Y'|\mathbf{C}_1,\mathbf{C}_2,do(\mathbf{S}),do(\mathbf{A})]\times Pr[\mathbf{C}_1,\mathbf{C}_2|do(\mathbf{S}),do(A)]\Big)\\
&&\text{Using Lemma~10 from the full version~\cite{fullversion}}\\
&&= \sum_{\mathbf{C}_1,\mathbf{C}_2}\Big(Pr[Y'|\mathbf{C}_1,\mathbf{C}_2,do(A)]\times Pr[\mathbf{C}_2|\mathbf{C}_1,do(\mathbf{S}),do(A)] Pr[\mathbf{C}_1|do(\mathbf{S}),do(A)]\Big)\\
&&\text{Since $Y'$is independent of $\mathbf{C}_2$ given $\mathbf{A}$ and $\mathbf{C}_1$}\\
&&= \sum_{\mathbf{C}_1,\mathbf{C}_2}\Big(Pr[Y'|\mathbf{C}_1,do(A)]Pr[\mathbf{C}_2|\mathbf{C}_1,do(\mathbf{S}),do(A)]\times  Pr[\mathbf{C}_1|do(\mathbf{S}),do(A)]\Big)
\end{eqnarray*}
\begin{eqnarray*}
&&\text{Summing $Pr[\mathbf{C}_2|\mathbf{C}_1,do(\mathbf{S}),do(A)]$ over $\mathbf{C}_2$}\\
&&= \sum_{\mathbf{C}_1}Pr[Y'|\mathbf{C}_1,do(A)]Pr[\mathbf{C}_1|do(\mathbf{S}),do(A)]=Pr[Y'|do(A)]\\
\end{eqnarray*}
}
This condition shows that  $\mathbf{A}\cup\mathbf{C}_1\cup \mathbf{C}_2$ ensure causal-fairness.
\end{proof}
This shows that the features $\mathbf{C}_1$ and $\mathbf{C}_2$ ensure causal fairness of the dataset. Using these results, we identify a closed form expression to identify all variables that ensure causal fairness. Note that whenever the trained classifier is not bayes optimal, $\mathbf{C}_1$ still ensure causal fairness but the effectiveness of $\mathbf{C}_2$ crucially relies on the optimality of the trained classifier.

\begin{theorem}
Consider a dataset $D$ with admissible variables $\mathbf{A}$, sensitive $\mathbf{S}$, a set of variables $\mathbf{X}$ with a target $Y$. A variable $X\in  \mathbf{X}$ is safe to be added along with $\mathbf{T}\cup \mathbf{A}$, where $\mathbf{T}\subseteq \mathbf{C}_1\cup\mathbf{C}_2\cup \mathbf{A}$  without violating causal fairness iff  (i) $(X\bigCI\mathbf{S}|{A})$ for some $A\subseteq \mathbf{A}$ or (ii) $(X\bigCI Y|\mathbf{C}',\mathbf{A})$, where $(\mathbf{C}'\bigCI \mathbf{S}|A)$ or (iii) $X$ is  not a descendant of $\mathbf{S}$ in $G_{\bar{A}}$, where $G_{\bar{\mathbf{A}}}$ is same as $G$ with  incoming edges of $\mathbf{A}$ removed.\label{thm:main}
\end{theorem}
\begin{proof}
\revb{Using Lemma~\ref{lem:c1} and \ref{lem:c2}, we can observe that all the variables $\mathbf{C}_1\cup \mathbf{C}_2$ such that $(\mathbf{C}_1\bigCI \mathbf{S}|{A})$, where $A\subseteq \mathbf{A}$ and $(\mathbf{C}_2\bigCI Y|\mathbf{C}_1,\mathbf{A})$ are safe to be added without worsening the fairness of the dataset. Now consider a variable $X$, which is not a descendant of $\mathbf{S}$ in $G_{\bar{\mathbf{A}}}$. All paths from $\mathbf{S}$ to $X$ are blocked when we intervene on $\mathbf{A}$ as all incoming edges of $\mathbf{A}$ are removed. Therefore it is safe to add $X$ without affecting causal fairness of the dataset.

To show the converse, when $X\nbigCI\mathbf{S}|A$, $\forall A\subseteq \mathbf{A}$ and $X\nbigCI Y|C',\mathbf{A}$ and $X$ is a descendant of $\mathbf{S}$ in $G_{\bar{\mathbf{A}}}$, then  we show that $X$ can worsen the fairness.
We can observe the following properties about $X$:
\begin{itemize}
    \item $(\mathbf{S}\nbigCI {X}|\mathbf{A})$ implies  there exists a path from $\mathbf{S}$ to ${X}$ that is unblocked given $\mathbf{A}$.
    \item  $(X\nbigCI {Y}|\mathbf{A},C')$ implies that $X$ is predicitve of $Y$ given the features $\mathbf{T}\subseteq \mathbf{C}_1\cup\mathbf{C}_2$. Therefore, there will be a direct edge from $X$ to the learned variable $Y'$. 
\end{itemize} 

If the paths from $\mathbf{S}$ to $X$ are unblocked in  $G_{\bar{\mathbf{A}}}$ then $\mathbf{S}$ to $X$ is unblocked when we intervene on ${\mathbf{A}}$. In this case, the path from $\mathbf{S}\rightarrow X\rightarrow Y'$ is unblocked and therefore $X$ is a biased variable that violates causal fairness of the dataset.}
\end{proof}
\texttt{SeqSel} captures variables that can be identified by performing CI tests. However, the last condition of Theorem~\ref{thm:main} requires intervention to identify other variables. Devising a set of CI tests to identify these variables is an interesting question for future work.
\begin{remark}
\reva{PC-algorithm~\cite{spirtes2000causation}, one of the most popular causal discovery techniques learn the causal graph structure from the data. } However, these techniques are known to work under specific modelling assumptions of the data and are highly inefficient. The number of CI tests required by such techniques is generally exponential in the number of input attributes.
\end{remark}
\noindent \textbf{Complexity:} Algorithm~\ref{alg:fairjoin} tests conditional independence (CI) of each variable with $\mathbf{S}$ and $Y$. In the worst case, it requires $O(2^{|\mathbf{A}|}n)$ CI tests to identify all the variables that do not worsen the fairness of $D$. In most realistic scenarios, $|\mathbf{A}|$ is a small constant, yielding overall complexity of $O(n)$, where $n$ is the number of features. Existing CI testing techniques can generate spurious correlations between independent variables for large values of $n$.    In the next section, we propose a group testing formulation that reduces this complexity to $O(\log n)$ tests, thereby improving its accuracy.

\subsection{Group Testing}
Group testing is an old technique that efficiently performs tests on a logarithmic number of groups of items rather than testing each item separately. It has  not been used in causal inference to identify independent variables. We use graphoid axioms to show the following two results for any collection of variables $\mathbf{X}$ and $Z$ justifying the correctness of group testing in our framework.
\begin{minipage}[t]{0.45\textwidth}
{\small
\begin{algorithm}[H]
   \caption{ \texttt{GrpSel} }
   \label{alg:modified}
\begin{algorithmic}[1]
   \State {\bfseries Input:} Variables $\mathbf{A}, \mathbf{S}, \mathbf{X}, Y$
   \State $\mathbf{C}_1\leftarrow \texttt{first\_phase}((\mathbf{A}, \mathbf{S}, \mathbf{X}_1,Y)$
  \State $\mathbf{C}_2\leftarrow \texttt{final\_candidates}((\mathbf{A}, \mathbf{S}, \mathbf{X}_1,Y, \mathbf{C}_1)$\\
   \Return  $\mathbf{C}_1 \cup \mathbf{C}_2$
   \end{algorithmic}
\end{algorithm}}
\end{minipage}

\begin{minipage}[t]{0.45\textwidth}
\small{
\begin{algorithm}[H]
   \caption{ \texttt{first\_phase} }
   \label{alg:first}
\begin{algorithmic}[1]
   \State {\bfseries Input:} Variables $\mathbf{A}, \mathbf{S}, \mathbf{X}, Y$
   \State $\mathbf{C}_1\leftarrow \phi$
\If {$\exists A\subseteq \mathbf{A} \text{ such that } (\mathbf{X}\bigCI \mathbf{S}|A)$} 
\State $\mathbf{C}_1\leftarrow \mathbf{X}$
\Else
\State $\mathbf{X}_1, \mathbf{X}_2 \leftarrow \texttt{random\_partition}(\mathbf{X})$
\State $\mathbf{C}_1\leftarrow \texttt{first\_phase}(\mathbf{A}, \mathbf{S}, \mathbf{X}_1,Y$) 
\State $\mathbf{C}_1\leftarrow \mathbf{C}_1\cup \texttt{first\_phase}(\mathbf{A}, \mathbf{S}, \mathbf{X}_2,Y$) 
\EndIf\\
   \Return  $\mathbf{C}_1$
   \end{algorithmic}
\end{algorithm}}
\end{minipage}

\begin{minipage}[t]{0.45\textwidth}
{\small
\begin{algorithm}[H]
   \caption{ \texttt{final\_candidates} }
   \label{alg:second}
\begin{algorithmic}[1]
   \State {\bfseries Input:} Variables $\mathbf{A}, \mathbf{S}, \mathbf{X}, Y, \mathbf{C}_1$
   \State $\mathbf{C}_2\leftarrow \phi$
\If {$(\mathbf{X}\bigCI {Y}|\mathbf{A},\mathbf{C}_1)$} 
\State $\mathbf{C}_2\leftarrow \mathbf{X}$
\Else
\State $\mathbf{X}_1, \mathbf{X}_2 \leftarrow \texttt{random\_partition}(\mathbf{X})$
\State $\mathbf{C}_2\leftarrow \texttt{final\_candidates}(\mathbf{A}, \mathbf{S}, \mathbf{X}_1,Y, \mathbf{C}_1$) 
\State $\mathbf{C}_2\leftarrow \mathbf{C}_2\cup \texttt{final\_candidates}(\mathbf{A}, \mathbf{S}, \mathbf{X}_2,Y, \mathbf{C}_2$) 
\EndIf\\
   \Return  $\mathbf{C}_2$
   \end{algorithmic}
\end{algorithm}}
\end{minipage}

\begin{lemma}\label{lem:chain1}
If $\exists X_i\in \mathbf{X}$ such that $X_1\nbigCI X_i|Z$ then  $(X_1\nbigCI \mathbf{X}\setminus\{X_1\}|Z) $ for some variables $X_1$ and $Z$.
\end{lemma}

\begin{lemma}\label{lem:chain2}
If    $(X_1\nbigCI \mathbf{X}\setminus X_1|Z) $ then $\exists X_i\in \mathbf{X}\setminus \{X_1\}$ such that $(X_1\nbigCI X_i|Z)$ for some $Z$.
\end{lemma}

These results yield the following two properties that make Algorithm~\ref{alg:fairjoin} more efficient.
\begin{itemize}
    \item If $(X_1\nbigCI X_2,X_3|Z)$ then $X_1\nbigCI X_2|Z$ or $X_1\nbigCI X_3|Z$
    \item If $(X_1\bigCI X_2,X_3|Z)$ then $X_1\bigCI X_2|Z$ and $X_2\bigCI X_3|Z$
\end{itemize}
Algorithm~\ref{alg:modified} presents an improved version of \texttt{SeqSel} that uses group testing to remove all the variables that do not satisfy the CI statements shown in Theorem~\ref{thm:main}. We call this approach \texttt{GrpSel}.  \texttt{GrpSel} operates in two phases, aiming to capture variables $\mathbf{C}_1$ and $\mathbf{C}_2$, respectively. The first phase (Algorithm~\ref{alg:first}) identifies the variables which do not capture any new information about sensitive variables given $A\subseteq \mathbf{A}$. It tests the CI between $\mathbf{S}$ and $\mathbf{X}$ given $A\subseteq\mathbf{A}$. If the variables are conditionally independent, then all the variables $\mathbf{X}$ are identified to maintain causal fairness. On the other hand, if the variables are conditionally dependent, the set $\mathbf{X}$ is partitioned into two equal partitions and \texttt{first\_phase} algorithm is called recursively for both the partitions.  Algorithm~\ref{alg:second}, performs the second phase to  identify the variables which are independent of the target variable $Y$ given $\mathbf{A}$ and $\mathbf{C}_1$. This algorithm operates similarly to \texttt{first\_phase} with a different CI test.

\noindent \textbf{Complexity.} Algorithm~\ref{alg:first} requires a total of $2^{|\mathbf{A}|}k\log n$ tests to identify all variables $X$ that satisfy $(\mathbf{S}\bigCI X|{A})$, where $k$ is the number of variables that do not satisfy the condition. The second phase requires $k'\log k$ tests to identify the variables that satisfy CI with $Y$ where $k'$ is the number of variables that do not satisfy the condition.
Therefore, \texttt{GrpSel} has better complexity when the total number of biased variables $k$ is $o(n/\log n)$.

\begin{figure*}[h]
    \centering
    \includegraphics[width=0.92\textwidth]{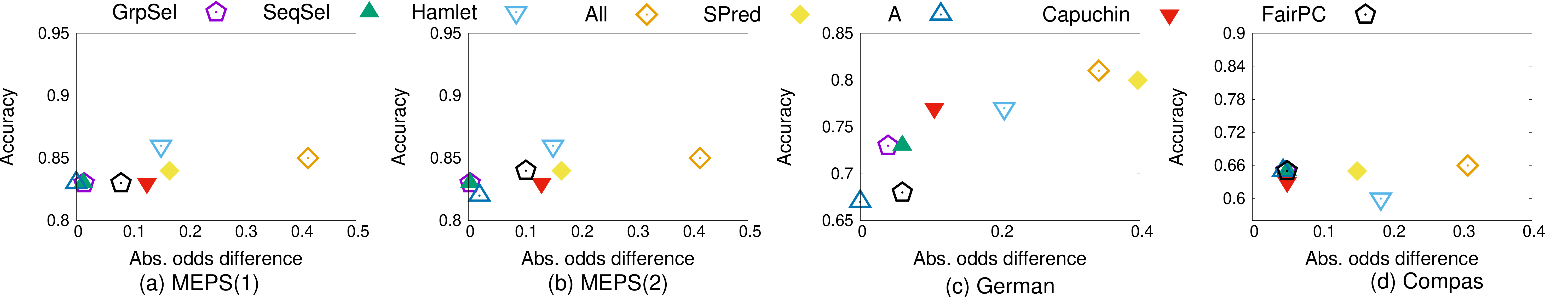}
    \vspace{-4mm}
    \caption{Classifier fairness and accuracy on MEPS, German, and Compas datasets.  \label{fig:quality1}}
    \vspace{-3mm}
\end{figure*}

\section{Experiments}
In this section, we empirically evaluate our technique along with baselines on real-world and synthetic datasets. We answer the following research questions. \textbf{Q1} Are \texttt{SeqSel} and \texttt{GrpSel} able to ensure causal fairness of the trained classifier? \textbf{Q2} How does the quality of classifier trained using different feature selection algorithms compare in terms of fairness and accuracy? \textbf{Q3} Is \texttt{GrpSel} effective in reducing the number of required CI tests?

\subsection{Setup}

\noindent \textbf{Datasets.} We consider the following  datasets.

\begin{itemize}
    \item \textit{Medical Expenditure }
(MEPS)~\cite{meps}: predict total number of hospital visits from patient medical information (Healthcare utilization is sometimes used as a proxy for allocating home care). We consider two variations denoted by MEPS(1) and MEPS(2). MEPS(1) considers `Arthritis diagnosis' as admissible and  MEPS(2) considers `Arthritis diagnosis' and `Mental health' as admissible. Race is considered sensitive. Contains 7915 training and 3100 test records.
\item \textit{German Credit}~\cite{german} 
applications. The account status is considered admissible and person's age is used as a sensitive attribute. Contains 800 training and 200 test records.
\item \textit{Compas}~\cite{compas}
: predict criminal recidivism from features such as the severity of the original crime. \revc{The number of prior convictions, age and severity of charge degree are taken as admissible and race as sensitive.} Contains 7200 samples.
\item \revc{\textit{Adult}~\cite{asuncion2007uci}
: predict income of individuals. Gender is considered sensitive and  hours per week, occupation, age, education are considered admissible. Contains $48k$ individuals.}
\item \textit{Synthetic}: a synthetically constructed dataset where a feature is constructed to be highly correlated to a sensitive feature with probability $p$. This dataset is used for understanding the effect of number of features and the fraction of noisy features on the complexity of our techniques.
\end{itemize}

\textbf{Baselines.} We consider the following baselines to identify a subset of features for the training task. 
\begin{enumerate}
    \item \texttt{A}: uses the variables in the admissible set.
    \item \texttt{ALL}: uses all features present in the dataset. 
    \item \texttt{Hamlet} \citep{kumar2016join}: uses heuristics to identify features which do not add value to the data set and can be ignored. 
    \item  \texttt{SPred}: learn a classifier using an exhaustive set of features to predict the sensitive attribute. Based on feature importance, we remove the highly predictive features.
    \item \revc{\texttt{Capuchin}~\cite{capuchin}: state-of-the-art in-processing technique that ensures causal fairness by adding or removing tuples.}
    \item \revc{\texttt{Fair-PC}: learns the causal graph using PC algorithm~\cite{spirtes2000causation} and uses it to infer features that ensure causal fairness.}
\end{enumerate}

\textbf{Experiment Setup.}
We evaluate accuracy and fairness of the trained classifier on the test set. To evaluate fairness, we measure conditional mutual information (CMI) and absolute odds difference \sainyam{calculated as the difference in false positive rate and true positive rate between the privileged and unprivileged groups}. \sainyam{
We consider the CMI and group fairness metric as a proxy because zero CMI implies causal fairness which further implies group fairness and can be easily evaluated from observed data~\citep{capuchin}.}
We use RCIT~\citep{strobl2019approximate} package in R for CI tests and logistic regression as the classifier.
\begin{figure}
\vspace{-2mm}
\includegraphics[width=0.95\columnwidth]{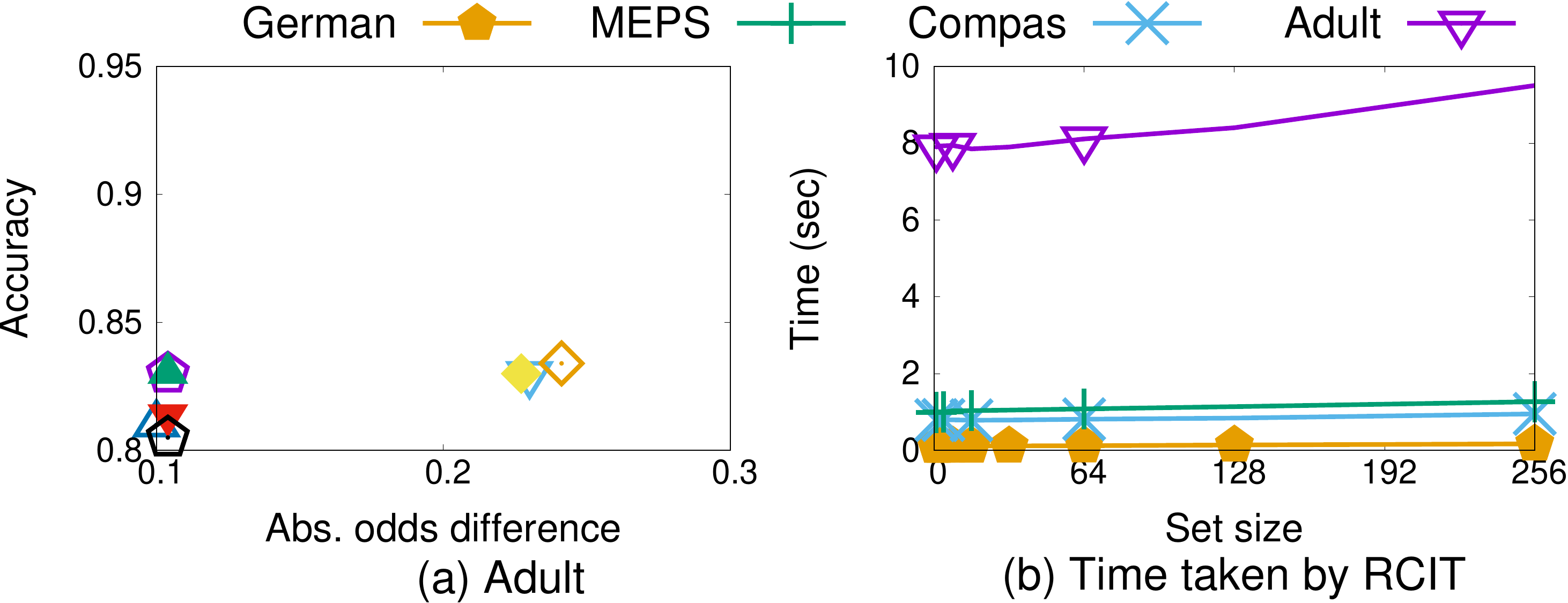}
\vspace{-3mm}
\caption{\revc{(a) Accuracy vs. Abs. odds difference (b) Running time comparison for varying conditioning set size.}}
\vspace{-5mm}
\label{fig:newplot}
\end{figure}

\begin{figure*}
  \begin{minipage}[b]{0.46\textwidth}
  \centering
    \includegraphics[width=0.92\textwidth]{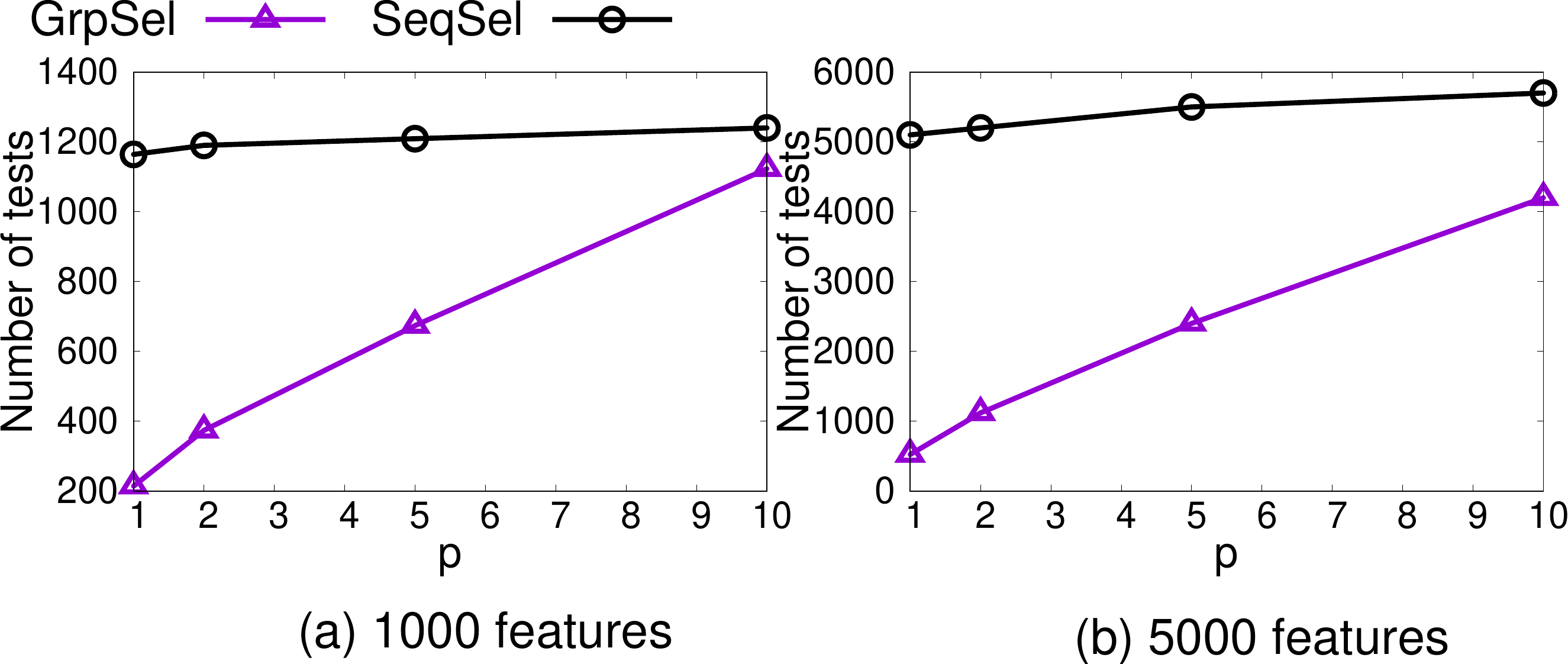}
    \vspace{-4mm}
    \caption{Total number of conditional independence tests vs. $p$, the percentage of biased variables.}
    \label{fig:varyp}
  \end{minipage}~
  \begin{minipage}[b]{0.46\textwidth}
     \includegraphics[width=0.92\textwidth]{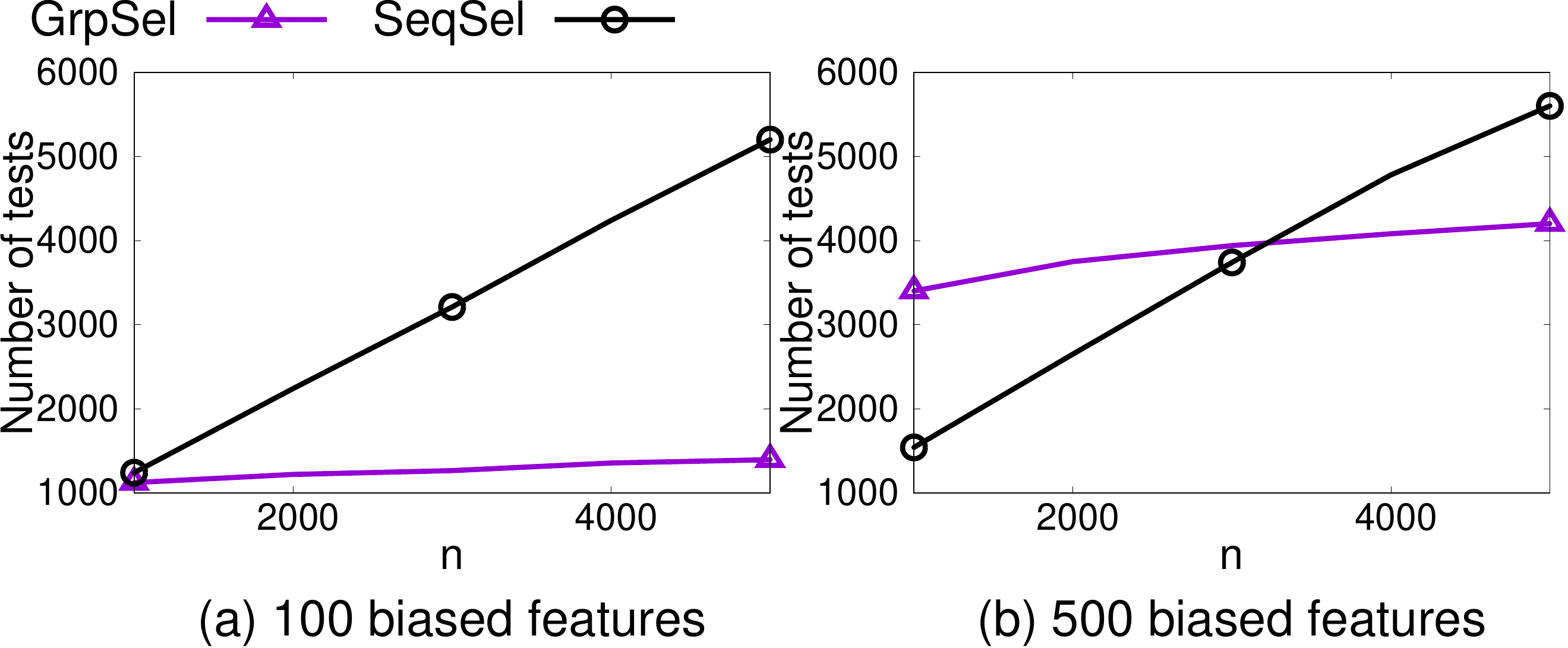}
     \vspace{-4mm}
    \caption{Total number of conditional independence tests vs. $n$ for a fixed number of biased variables.}
    \label{fig:varyn}
  \end{minipage}
  \vspace{-5mm}
\end{figure*}
\subsection{Solution Quality\label{sec:realexp}}
Figure~\ref{fig:quality1} compares the accuracy of the classifier trained with the features identified by our baselines along with its fairness. \texttt{ALL} learns the most accurate classifier as compared to all other techniques. However, it achieves the highest odds difference and hence worst fairness with respect to the sensitive attribute of the dataset. \texttt{A} maintains high fairness but achieves quite low accuracy as compared to \texttt{SeqSel} and \texttt{GrpSel}. \texttt{Hamlet} does not identify features that are highly correlated with sensitive attributes and does not improve its fairness.   \texttt{SPred} identifies a few features that capture sensitive information but is unable to identify all such features. Hence, it does not improve the fairness of the classifier as compared to \texttt{GrpSel}. \revc{\texttt{Capuchin} and \texttt{FairPC} are able to improve fairness as compared to \texttt{ALL}  but performs worse than \texttt{GrpSel} and \texttt{SeqSel}. However, accuracy of the  learnt classifier is lower for \texttt{FairPC} than \texttt{Capuchin}, \texttt{SeqSel}, and \texttt{GrpSel}.  }\texttt{SeqSel} and \texttt{GrpSel} maintain high fairness with respect to various metrics of fairness without much loss in accuracy. \reva{We calculated feature importance of identified attributes and identified that a number of attributes identified in the second phase of  our algorithm have non-zero feature importance and contribute towards classifier prediction.}

For MEPS and German datasets, \texttt{GrpSel} and \texttt{SeqSel} are able to identify features that mitigate the bias and do not lose much in classifier accuracy. However, all other techniques have higher bias against the protected attribute on Compas. \revc{In this case, we observe that the admissible feature is correlated to the sensitive attribute, affecting the fairness of the trained classifier. }
We empirically swept the p-value threshold from 0.01 to 0.05, and results are stable and do not impact its performance. As an example, the accuracy of the trained classifier was 0.83-0.84 on MEPS and within 0.73-0.76 on German on varying the thresholds. 
We observed similar behavior on changing  the classifier from logistic regression to random forest.

Table~\ref{tab:cmi} compares the conditional mutual information between the learnt variable $Y'$ (according to \texttt{GrpSel}) and target $Y$ with $\mathbf{S}$ given $\mathbf{A}$.\footnote{Some mutual information values were slightly negative and were truncated to $0$ as suggested by~\citet{mukherjee2019ccmi}.} Across all datasets, $Y'$ is independent of $\mathbf{S}$ even though the original target variable $Y$ was unfair. This experiment validates the efficacy of our techniques to identify features that ensure fairness and get rid of the biased features.

\begin{minipage}[t]{0.46\textwidth}
\begin{table}[H]
\vspace{-6mm}
\begin{center}
\footnotesize
 \begin{tabular}{|c |c| c|c|c|c|c|} 
 \hline
  &\multicolumn{2}{c|}{CMI}&&&\multicolumn{2}{c|}{Number of tests}\\
  \hline
 Dataset & CMI$(\mathbf{S},Y' | \mathbf{A})$  & CMI$(\mathbf{S},Y | \mathbf{A})$&&&\texttt{SeqSel}&\texttt{GrpSel} \\

 \hline\hline
 MEPS(1)   & 0.0 & 0.015  && MEPS(1) &343 &247 \\ 
 \hline
  MEPS(2) & 0.0 & 0.014 &&MEPS(2) &420&390 \\ 
 \hline
 German   & 0.002 & 0.018 &&German&525 &81  \\
 \hline
 Compas & 0.0 & 0.01 &&Compas&257&83 \\ 
 \hline
 Adult & 0.01 & 0.03 && Adult&125&23 \\
 \hline
\end{tabular}
\end{center}
\caption{Conditional Mutual Information~\citep{mukherjee2019ccmi} and number of CI tests required for each dataset}
\vspace{-5mm}
\label{tab:cmi}
\end{table}
\end{minipage}

\smallskip \noindent \textbf{Model Selection.} We tested these pipelines by training other ML algorithms like random forest and Adaboost classifier. Across all datasets, we observe that \texttt{SeqSel} and \texttt{GrpSel} maintain fairness of the trained classifier while maintaining high accuracy.

\subsection{Synthetic Data}
In this experiment, we tested the causal fairness metric by simulating interventions presented in Definition~\ref{def:causal-fair} and compared with ground truth.
We evaluate \texttt{GrpSel} and \texttt{SeqSel} on multiple synthetic datasets generated using causal graphs of varied sizes (1000, 3000 and 5000). 
Across all datasets, we observed that \texttt{SeqSel} and \texttt{GrpSel} identified majority of the variables that ensure causal fairness. However, other baselines were not able to identify all the biased features, thereby leading to biased datasets. 

\textbf{Complexity.} The total number of CI tests required by  \texttt{SeqSel} and \texttt{GrpSel} are shown in Table~2.  \texttt{GrpSel} requires fewer tests than \texttt{SeqSel} across all datasets. Since all these datasets contain fewer than 1000 features, the improvement is not very significant. 
To understand the difference in complexity of the two techniques, we perform an extensive simulation study  by varying the total number of features and the fraction of biased variables.

Figure~\ref{fig:varyn} compares the total number of CI tests required to identify variables that ensure causal fairness. With the increase in total number of features ($n$), the number of tests required by \texttt{SeqSel} grows linearly. However, the growth of \texttt{GrpSel} is sub-linear and requires fewer tests than \texttt{SeqSel} for larger $n$. This result is coherent with our theoretical analysis of $O(n)$ tests for \texttt{SeqSel} and $O(k\log n)$ for \texttt{GrpSel}, where $k$ is the  number of biased variables and $n$ is the total number of features in the dataset.

\textbf{Effect of $p$.}
Figure~\ref{fig:varyp} compares \texttt{GrpSel} and \texttt{SeqSel} as a function of the total fraction of biased variables in the dataset.  \texttt{SeqSel}'s complexity is driven by the total number of features irrespective of the number of biased features.  However, the tests required by \texttt{GrpSel} are dependent linearly on $p$. This experiment confirms the benefit of using group testing  when the total number of biased variables are fewer than ${(\log n)}/{n}$.

\noindent \textbf{\revc{Advantages of Group-testing\label{sec:advgrp}}} \revc{We now evaluate the benefits of using group-testing based technique for feature selection.  We generated a synthetic dataset containing $1000$ records and increased the number of features (denoted by $t$) from $100$ to $1000$ in increments of $100$. We tested the correctness of \texttt{GrpSel}'s output with the ground-truth calculated from the causal graph. We observed that around $5$ attributes that are independent of $S$ are dropped by \texttt{SeqSel} when $t=500$. The spuriousness increases to $\approx 47$ features when $t=1000$. On the other hand, \texttt{GrpSel} did not return any spurious correlation for $t\leq 900$ features and returned less than $5$ spurious features when $t=1000$. This experiment demonstrates that group-testing can reduce the chances of getting a spurious output.

}
\subsection{Robustness\label{sec:robustness}}
 \revc{ In this experiment, we changed  the test data by modifying the effect of sensitive attribute on the target variable through specific attributes (by changing edge weights of the causal graph). This data distribution shift did not affect the performance of \texttt{GrpSel} or \texttt{SeqSel} and both techniques achieved $0$ absolute odds difference. In contrast, prior pre-processing techniques led to an increase in absolute odds difference of upto $15\%$. The evaluation demonstrated the weakness of pre-processing techniques to generalize to settings beyond the data distribution of the repaired training dataset. Prior work has referred to it as over-fitting with respect to fairness~\cite{diazautomated}.}

\noindent \textbf{Running time. } Figure~\ref{fig:newplot}(b) compares the running time of CI test run using RCIT package for varying size of the conditioning set. This experiment shows that the running time increases linearly with increasing set size but the gradient is very slow. For example, the running time for the adult dataset increases from $8$ sec to less than $10$ sec when the conditioning set size increases from $1$ to $256$. Therefore, performing a CI test with groups of features is effective.

Among all techniques, we observe that \texttt{GrpSel} and \texttt{SeqSel} execute within 10 minutes for all real-world datasets, and it takes around 1 minute to train a classifier. Therefore, our techniques  learns a fair classifier in less than 11 minutes across all datasets. 

\section{Related Work}
To the best of our knowledge, there is very little related work on discrimination-aware or fair feature selection.
One of the recent papers on feature construction and exploration~\cite{diazautomated} has studied the problem of constructing new features that can help improve prediction without affecting fairness.
\citet{GrgricHlacaZGW2018} use human moral judgements of different properties of features (volitionality, reliability, privacy, and relevance) as the starting point for feature selection. Although they cite causal fairness definitions as the basis for feature relevance, they do not use the data to quantify this relevance. \sainyam{\citet{capuchin} consider causal fairness to change the input data distribution as opposed to identification of a small set of features that ensure causal fairness.} \citet{DuttaVMDG2020} start with the causal fairness perspective as well and also use tools from information theory, but use partial information decomposition to partition the information contained in the features into exempt and non-exempt portions; the goal is not feature subset selection, but gaining insight into different types of discrimination.
\citet{nabi2018fair} considered  causal pathways to identify discrimination and then train a fair classifier assuming full knowledge of the underlying causal graph. \citet{zhang2016causal} consider causal definitions of fairness and devise algorithms that repair the dataset to ensure fairness.
\citet{noriega2018active} and its followup \citep{BakkerNTSVP2019} examine an active feature acquisition paradigm from the perspective of fairness but do not study the causal notion of fairness. 

\section{Conclusion}
In this paper, we have tackled the problem of data integration --- joining additional features to an initially given dataset --- while not introducing additional unwanted bias against protected groups. We have utilized the formalism of causal fairness and \texttt{do-calculus} to develop an algorithm for adding variables that is theoretically-guaranteed not to make fairness worse. We have enhanced this algorithm using group testing to make it more efficient (the first use of group testing in such a setting) and shown its efficacy on several datasets. The extension of our techniques for active learning or online setting are interesting questions for future work.

\bibliographystyle{ACM-Reference-Format}
\bibliography{fairjoin}


\begin{thebibliography}{56}


\ifx \showCODEN    \undefined \def \showCODEN     #1{\unskip}     \fi
\ifx \showDOI      \undefined \def \showDOI       #1{#1}\fi
\ifx \showISBNx    \undefined \def \showISBNx     #1{\unskip}     \fi
\ifx \showISBNxiii \undefined \def \showISBNxiii  #1{\unskip}     \fi
\ifx \showISSN     \undefined \def \showISSN      #1{\unskip}     \fi
\ifx \showLCCN     \undefined \def \showLCCN      #1{\unskip}     \fi
\ifx \shownote     \undefined \def \shownote      #1{#1}          \fi
\ifx \showarticletitle \undefined \def \showarticletitle #1{#1}   \fi
\ifx \showURL      \undefined \def \showURL       {\relax}        \fi
\providecommand\bibfield[2]{#2}
\providecommand\bibinfo[2]{#2}
\providecommand\natexlab[1]{#1}
\providecommand\showeprint[2][]{arXiv:#2}

\bibitem[\protect\citeauthoryear{??}{ger}{2013}]%
        {german}
 \bibinfo{year}{2013}\natexlab{}.
\newblock \showarticletitle{Uci machine learning repository.}
\newblock
  \bibinfo{journal}{\emph{\url{https://archive.ics.uci.edu/ml/datasets/Statlog+\%28German+Credit+Data\%29}}}
  (\bibinfo{year}{2013}).
\newblock


\bibitem[\protect\citeauthoryear{??}{mep}{2016}]%
        {meps}
 \bibinfo{year}{2016}\natexlab{}.
\newblock \showarticletitle{Medical Expenditure Panel Survey.}
\newblock \bibinfo{journal}{\emph{\url{https://meps.ahrq.gov/mepsweb/}}}
  (\bibinfo{year}{2016}).
\newblock


\bibitem[\protect\citeauthoryear{Asuncion and Newman}{Asuncion and
  Newman}{2007}]%
        {asuncion2007uci}
\bibfield{author}{\bibinfo{person}{Arthur Asuncion} {and}
  \bibinfo{person}{David Newman}.} \bibinfo{year}{2007}\natexlab{}.
\newblock \bibinfo{title}{UCI machine learning repository}.
\newblock
\newblock


\bibitem[\protect\citeauthoryear{Bakker, Noriega~Campero, Tu, Sattigeri,
  Varshney, and Pentland}{Bakker et~al\mbox{.}}{2019}]%
        {BakkerNTSVP2019}
\bibfield{author}{\bibinfo{person}{Michiel Bakker}, \bibinfo{person}{Alejandro
  Noriega~Campero}, \bibinfo{person}{Duy~Patrick Tu}, \bibinfo{person}{Prasanna
  Sattigeri}, \bibinfo{person}{Kush~R. Varshney}, {and} \bibinfo{person}{Alex
  Pentland}.} \bibinfo{year}{2019}\natexlab{}.
\newblock \showarticletitle{On Fairness in Budget-Constrained Decision Making}.
  In \bibinfo{booktitle}{\emph{KDD Workshop on Explainable Artificial
  Intelligence}}.
\newblock


\bibitem[\protect\citeauthoryear{Barocas, Hardt, and Narayanan}{Barocas
  et~al\mbox{.}}{2020}]%
        {BarocasHN2020}
\bibfield{author}{\bibinfo{person}{Solon Barocas}, \bibinfo{person}{Moritz
  Hardt}, {and} \bibinfo{person}{Arvind Narayanan}.}
  \bibinfo{year}{2020}\natexlab{}.
\newblock \bibinfo{booktitle}{\emph{Fairness and Machine Learning: Limitations
  and Opportunities}}.
\newblock \bibinfo{address}{https://fairmlbook.org}.
\newblock


\bibitem[\protect\citeauthoryear{Calders and Verwer}{Calders and
  Verwer}{2010}]%
        {calders2010three}
\bibfield{author}{\bibinfo{person}{Toon Calders} {and} \bibinfo{person}{Sicco
  Verwer}.} \bibinfo{year}{2010}\natexlab{}.
\newblock \showarticletitle{Three naive bayes approaches for
  discrimination-free classification}.
\newblock \bibinfo{journal}{\emph{Data mining and knowledge discovery}}
  \bibinfo{volume}{21}, \bibinfo{number}{2} (\bibinfo{year}{2010}),
  \bibinfo{pages}{277--292}.
\newblock


\bibitem[\protect\citeauthoryear{Calmon, Wei, Vinzamuri, Ramamurthy, and
  Varshney}{Calmon et~al\mbox{.}}{2017}]%
        {calmon2017optimized}
\bibfield{author}{\bibinfo{person}{Flavio~P Calmon}, \bibinfo{person}{Dennis
  Wei}, \bibinfo{person}{Bhanukiran Vinzamuri},
  \bibinfo{person}{Karthikeyan~Natesan Ramamurthy}, {and}
  \bibinfo{person}{Kush~R Varshney}.} \bibinfo{year}{2017}\natexlab{}.
\newblock \showarticletitle{Optimized pre-processing for discrimination
  prevention}. In \bibinfo{booktitle}{\emph{Proceedings of the 31st
  International Conference on Neural Information Processing Systems}}.
  \bibinfo{pages}{3995--4004}.
\newblock


\bibitem[\protect\citeauthoryear{Celis, Huang, Keswani, and Vishnoi}{Celis
  et~al\mbox{.}}{2019}]%
        {celis2019classification}
\bibfield{author}{\bibinfo{person}{L~Elisa Celis}, \bibinfo{person}{Lingxiao
  Huang}, \bibinfo{person}{Vijay Keswani}, {and} \bibinfo{person}{Nisheeth~K
  Vishnoi}.} \bibinfo{year}{2019}\natexlab{}.
\newblock \showarticletitle{Classification with fairness constraints: A
  meta-algorithm with provable guarantees}. In
  \bibinfo{booktitle}{\emph{Proceedings of the conference on fairness,
  accountability, and transparency}}. \bibinfo{pages}{319--328}.
\newblock


\bibitem[\protect\citeauthoryear{Chiappa}{Chiappa}{2019}]%
        {chiappa2019path}
\bibfield{author}{\bibinfo{person}{Silvia Chiappa}.}
  \bibinfo{year}{2019}\natexlab{}.
\newblock \showarticletitle{Path-specific counterfactual fairness}. In
  \bibinfo{booktitle}{\emph{Proceedings of the AAAI Conference on Artificial
  Intelligence}}, Vol.~\bibinfo{volume}{33}. \bibinfo{pages}{7801--7808}.
\newblock


\bibitem[\protect\citeauthoryear{Chiappa and Isaac}{Chiappa and Isaac}{2018}]%
        {chiappa2018causal}
\bibfield{author}{\bibinfo{person}{Silvia Chiappa} {and}
  \bibinfo{person}{William~S Isaac}.} \bibinfo{year}{2018}\natexlab{}.
\newblock \showarticletitle{A causal Bayesian networks viewpoint on fairness}.
  In \bibinfo{booktitle}{\emph{IFIP International Summer School on Privacy and
  Identity Management}}. Springer, \bibinfo{pages}{3--20}.
\newblock


\bibitem[\protect\citeauthoryear{d'Alessandro, O'Neil, and
  LaGatta}{d'Alessandro et~al\mbox{.}}{2017}]%
        {DalessandroOL2017}
\bibfield{author}{\bibinfo{person}{Brian d'Alessandro}, \bibinfo{person}{Cathy
  O'Neil}, {and} \bibinfo{person}{Tom LaGatta}.}
  \bibinfo{year}{2017}\natexlab{}.
\newblock \showarticletitle{Conscientious Classification: A Data Scientist's
  Guide to Discrimination-Aware Classification}.
\newblock \bibinfo{journal}{\emph{Big Data}} \bibinfo{volume}{5},
  \bibinfo{number}{2} (\bibinfo{date}{June} \bibinfo{year}{2017}),
  \bibinfo{pages}{120--134}.
\newblock


\bibitem[\protect\citeauthoryear{Diaz, Neutatz, and Abedjan}{Diaz
  et~al\mbox{.}}{2021}]%
        {diazautomated}
\bibfield{author}{\bibinfo{person}{Ricardo~Salazar Diaz},
  \bibinfo{person}{Felix Neutatz}, {and} \bibinfo{person}{Ziawasch Abedjan}.}
  \bibinfo{year}{2021}\natexlab{}.
\newblock \showarticletitle{Automated Feature Engineering for Algorithmic
  Fairness}.
\newblock \bibinfo{journal}{\emph{PVLDB}} (\bibinfo{year}{2021}).
\newblock


\bibitem[\protect\citeauthoryear{Dutta, Venkatesh, Mardziel, Datta, and
  Grover}{Dutta et~al\mbox{.}}{2020}]%
        {DuttaVMDG2020}
\bibfield{author}{\bibinfo{person}{Sanghamitra Dutta}, \bibinfo{person}{Praveen
  Venkatesh}, \bibinfo{person}{Piotr Mardziel}, \bibinfo{person}{Anupam Datta},
  {and} \bibinfo{person}{Pulkit Grover}.} \bibinfo{year}{2020}\natexlab{}.
\newblock \showarticletitle{An Information-Theoretic Quantification of
  Discrimination with Exempt Features}. In
  \bibinfo{booktitle}{\emph{Proceedings of the AAAI Conference on Artificial
  Intelligence}}.
\newblock


\bibitem[\protect\citeauthoryear{Feldman, Friedler, Moeller, Scheidegger, and
  Venkatasubramanian}{Feldman et~al\mbox{.}}{2015}]%
        {feldman2015certifying}
\bibfield{author}{\bibinfo{person}{Michael Feldman}, \bibinfo{person}{Sorelle~A
  Friedler}, \bibinfo{person}{John Moeller}, \bibinfo{person}{Carlos
  Scheidegger}, {and} \bibinfo{person}{Suresh Venkatasubramanian}.}
  \bibinfo{year}{2015}\natexlab{}.
\newblock \showarticletitle{Certifying and removing disparate impact}. In
  \bibinfo{booktitle}{\emph{proceedings of the 21th ACM SIGKDD international
  conference on knowledge discovery and data mining}}.
  \bibinfo{pages}{259--268}.
\newblock


\bibitem[\protect\citeauthoryear{Galhotra, Khurana, Hassanzadeh, Srinivas,
  Samulowitz, and Qi}{Galhotra et~al\mbox{.}}{2019}]%
        {galhotraautomated}
\bibfield{author}{\bibinfo{person}{Sainyam Galhotra}, \bibinfo{person}{Udayan
  Khurana}, \bibinfo{person}{Oktie Hassanzadeh}, \bibinfo{person}{Kavitha
  Srinivas}, \bibinfo{person}{Horst Samulowitz}, {and} \bibinfo{person}{Miao
  Qi}.} \bibinfo{year}{2019}\natexlab{}.
\newblock \showarticletitle{Automated Feature Enhancement for Predictive
  Modeling using External Knowledge}.
\newblock \bibinfo{journal}{\emph{ICDM}} (\bibinfo{year}{2019}).
\newblock


\bibitem[\protect\citeauthoryear{Galhotra, Shanmugam, Sattigeri, and
  Varshney}{Galhotra et~al\mbox{.}}{[n.d.]}]%
        {fullversion}
\bibfield{author}{\bibinfo{person}{Sainyam Galhotra},
  \bibinfo{person}{Karthikeyan Shanmugam}, \bibinfo{person}{Prasanna
  Sattigeri}, {and} \bibinfo{person}{Kush~R. Varshney}.}
  \bibinfo{year}{[n.d.]}\natexlab{}.
\newblock \bibinfo{title}{Causal Feature Selection for Algorithmic Fairness,
  {arXiv} \url{https://arxiv.org/abs/2006.06053}}.
\newblock
\newblock


\bibitem[\protect\citeauthoryear{Galhotra, Shanmugam, Sattigeri, and
  Varshney}{Galhotra et~al\mbox{.}}{2021}]%
        {galhotra2021interventional}
\bibfield{author}{\bibinfo{person}{Sainyam Galhotra},
  \bibinfo{person}{Karthikeyan Shanmugam}, \bibinfo{person}{Prasanna
  Sattigeri}, {and} \bibinfo{person}{Kush~R Varshney}.}
  \bibinfo{year}{2021}\natexlab{}.
\newblock \showarticletitle{Interventional Fairness with Indirect Knowledge of
  Unobserved Protected Attributes}.
\newblock \bibinfo{journal}{\emph{Entropy}} \bibinfo{volume}{23},
  \bibinfo{number}{12} (\bibinfo{year}{2021}), \bibinfo{pages}{1571}.
\newblock


\bibitem[\protect\citeauthoryear{Grgi\'{c}-Hla\v{c}a, Zafar, Gummadi, and
  Weller}{Grgi\'{c}-Hla\v{c}a et~al\mbox{.}}{2018}]%
        {GrgricHlacaZGW2018}
\bibfield{author}{\bibinfo{person}{Nina Grgi\'{c}-Hla\v{c}a},
  \bibinfo{person}{Muhammad~Bilal Zafar}, \bibinfo{person}{Krishna~P. Gummadi},
  {and} \bibinfo{person}{Adrian Weller}.} \bibinfo{year}{2018}\natexlab{}.
\newblock \showarticletitle{Beyond Distributive Fairness in Algorithmic
  Decision Making: Feature Selection for Procedurally Fair Learning}. In
  \bibinfo{booktitle}{\emph{Proceedings of the AAAI Conference on Artificial
  Intelligence}}. \bibinfo{pages}{51--60}.
\newblock


\bibitem[\protect\citeauthoryear{Hall}{Hall}{1999}]%
        {hall1999correlation}
\bibfield{author}{\bibinfo{person}{Mark~Andrew Hall}.}
  \bibinfo{year}{1999}\natexlab{}.
\newblock \showarticletitle{Correlation-based feature selection for machine
  learning}.
\newblock  (\bibinfo{year}{1999}).
\newblock


\bibitem[\protect\citeauthoryear{Hardt, Price, and Srebro}{Hardt
  et~al\mbox{.}}{2016}]%
        {hardt2016equality}
\bibfield{author}{\bibinfo{person}{Moritz Hardt}, \bibinfo{person}{Eric Price},
  {and} \bibinfo{person}{Nati Srebro}.} \bibinfo{year}{2016}\natexlab{}.
\newblock \showarticletitle{Equality of opportunity in supervised learning}.
\newblock \bibinfo{journal}{\emph{Advances in neural information processing
  systems}}  \bibinfo{volume}{29} (\bibinfo{year}{2016}),
  \bibinfo{pages}{3315--3323}.
\newblock


\bibitem[\protect\citeauthoryear{Holstein, Wortman~Vaughan, Daum\'{e},
  Dud\'{\i}k, and Wallach}{Holstein et~al\mbox{.}}{2019}]%
        {HolsteinWDDW2019}
\bibfield{author}{\bibinfo{person}{Kenneth Holstein}, \bibinfo{person}{Jennifer
  Wortman~Vaughan}, \bibinfo{person}{Hal Daum\'{e}, {III}},
  \bibinfo{person}{Miroslav Dud\'{\i}k}, {and} \bibinfo{person}{Hanna
  Wallach}.} \bibinfo{year}{2019}\natexlab{}.
\newblock \showarticletitle{Improving Fairness in Machine Learning Systems:
  What Do Industry Practitioners Need?}. In
  \bibinfo{booktitle}{\emph{Proceedings of the CHI Conference on Human Factors
  in Computing Systems}}. \bibinfo{pages}{600}.
\newblock


\bibitem[\protect\citeauthoryear{Huang and Valtorta}{Huang and
  Valtorta}{2012}]%
        {huang2012pearl}
\bibfield{author}{\bibinfo{person}{Yimin Huang} {and} \bibinfo{person}{Marco
  Valtorta}.} \bibinfo{year}{2012}\natexlab{}.
\newblock \showarticletitle{Pearl's calculus of intervention is complete}.
\newblock \bibinfo{journal}{\emph{arXiv preprint arXiv:1206.6831}}
  (\bibinfo{year}{2012}).
\newblock


\bibitem[\protect\citeauthoryear{Ingold and Soper}{Ingold and Soper}{2016}]%
        {zip}
\bibfield{author}{\bibinfo{person}{David Ingold} {and} \bibinfo{person}{Spencer
  Soper}.} \bibinfo{year}{2016}\natexlab{}.
\newblock \showarticletitle{Amazon doesn’t consider the race of its
  customers. should it?}
\newblock \bibinfo{journal}{\emph{Bloomberg}} (\bibinfo{year}{2016}).
\newblock


\bibitem[\protect\citeauthoryear{Jeff~Larson and Angwin.}{Jeff~Larson and
  Angwin.}{2016}]%
        {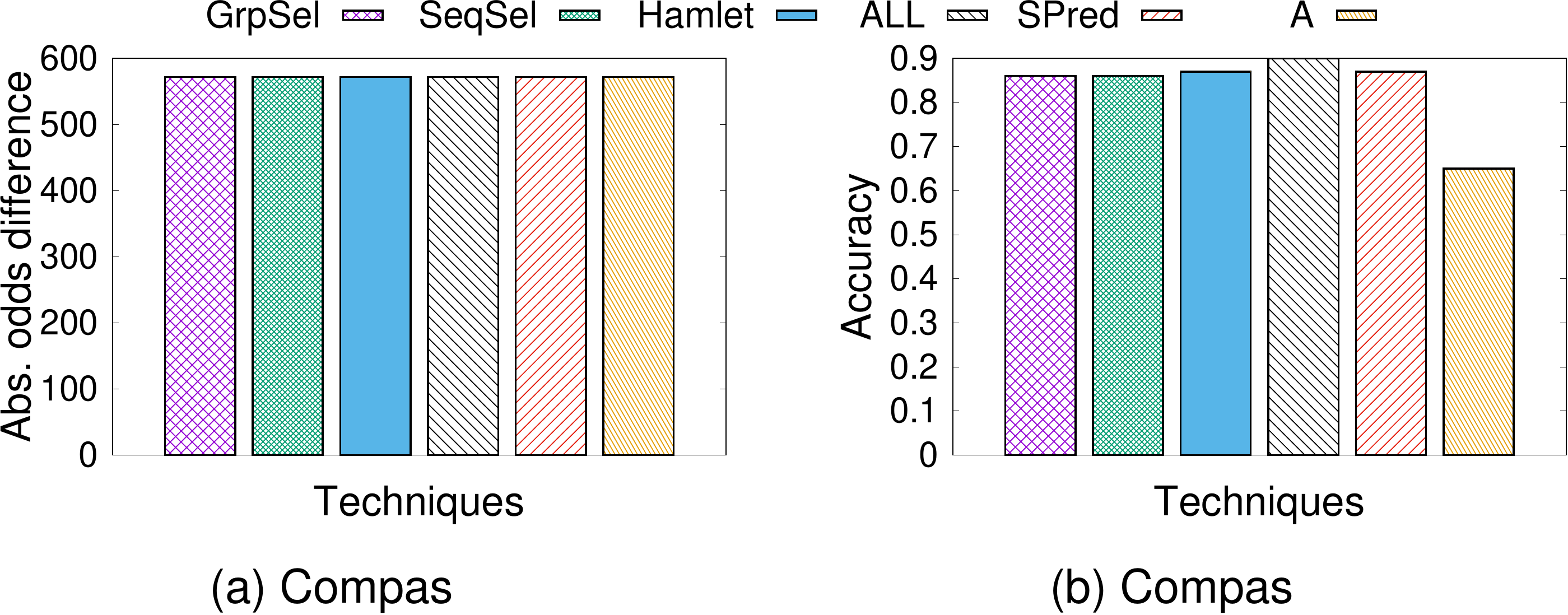}
\bibfield{author}{\bibinfo{person}{Lauren~Kirchner Jeff~Larson, Surya~Mattu}
  {and} \bibinfo{person}{Julia Angwin.}} \bibinfo{year}{2016}\natexlab{}.
\newblock \showarticletitle{How we analyzed the compas recidivism algorithm.}
\newblock \bibinfo{journal}{\emph{ProPublica}} (\bibinfo{year}{2016}).
\newblock


\bibitem[\protect\citeauthoryear{Jiang, Pacchiano, Stepleton, Jiang, and
  Chiappa}{Jiang et~al\mbox{.}}{2019}]%
        {jiang2019wasserstein}
\bibfield{author}{\bibinfo{person}{Ray Jiang}, \bibinfo{person}{Aldo
  Pacchiano}, \bibinfo{person}{Tom Stepleton}, \bibinfo{person}{Heinrich
  Jiang}, {and} \bibinfo{person}{Silvia Chiappa}.}
  \bibinfo{year}{2019}\natexlab{}.
\newblock \showarticletitle{Wasserstein fair classification}.
\newblock \bibinfo{journal}{\emph{arXiv preprint arXiv:1907.12059}}
  (\bibinfo{year}{2019}).
\newblock


\bibitem[\protect\citeauthoryear{Jo and Gebru}{Jo and Gebru}{2020}]%
        {JoG2020}
\bibfield{author}{\bibinfo{person}{Eun~Seo Jo} {and} \bibinfo{person}{Timnit
  Gebru}.} \bibinfo{year}{2020}\natexlab{}.
\newblock \showarticletitle{Lessons from Archives: Strategies for Collecting
  Sociocultural Data in Machine Learning}. In
  \bibinfo{booktitle}{\emph{Proceedings of the ACM Conference on Fairness,
  Accountability, and Transparency}}.
\newblock


\bibitem[\protect\citeauthoryear{Kamiran and Calders}{Kamiran and
  Calders}{2012}]%
        {kamiran2012data}
\bibfield{author}{\bibinfo{person}{Faisal Kamiran} {and} \bibinfo{person}{Toon
  Calders}.} \bibinfo{year}{2012}\natexlab{}.
\newblock \showarticletitle{Data preprocessing techniques for classification
  without discrimination}.
\newblock \bibinfo{journal}{\emph{Knowledge and Information Systems}}
  \bibinfo{volume}{33}, \bibinfo{number}{1} (\bibinfo{year}{2012}),
  \bibinfo{pages}{1--33}.
\newblock


\bibitem[\protect\citeauthoryear{Kamishima, Akaho, Asoh, and Sakuma}{Kamishima
  et~al\mbox{.}}{2012}]%
        {kamishima2012fairness}
\bibfield{author}{\bibinfo{person}{Toshihiro Kamishima},
  \bibinfo{person}{Shotaro Akaho}, \bibinfo{person}{Hideki Asoh}, {and}
  \bibinfo{person}{Jun Sakuma}.} \bibinfo{year}{2012}\natexlab{}.
\newblock \showarticletitle{Fairness-aware classifier with prejudice remover
  regularizer}. In \bibinfo{booktitle}{\emph{Joint European Conference on
  Machine Learning and Knowledge Discovery in Databases}}. Springer,
  \bibinfo{pages}{35--50}.
\newblock


\bibitem[\protect\citeauthoryear{Khademi and Honavar}{Khademi and
  Honavar}{2019}]%
        {khademi2019algorithmic}
\bibfield{author}{\bibinfo{person}{Aria Khademi} {and} \bibinfo{person}{Vasant
  Honavar}.} \bibinfo{year}{2019}\natexlab{}.
\newblock \showarticletitle{Algorithmic Bias in Recidivism Prediction: A Causal
  Perspective}.
\newblock \bibinfo{journal}{\emph{arXiv}}.
\newblock
\showeprint{1911.10640}


\bibitem[\protect\citeauthoryear{Khademi, Lee, Foley, and Honavar}{Khademi
  et~al\mbox{.}}{2019}]%
        {khademi2019fairness}
\bibfield{author}{\bibinfo{person}{Aria Khademi}, \bibinfo{person}{Sanghack
  Lee}, \bibinfo{person}{David Foley}, {and} \bibinfo{person}{Vasant Honavar}.}
  \bibinfo{year}{2019}\natexlab{}.
\newblock \showarticletitle{Fairness in algorithmic decision making: An
  excursion through the lens of causality}. In \bibinfo{booktitle}{\emph{The
  World Wide Web Conference}}. \bibinfo{pages}{2907--2914}.
\newblock


\bibitem[\protect\citeauthoryear{Khurana, Turaga, Samulowitz, and
  Parthasrathy}{Khurana et~al\mbox{.}}{2016}]%
        {khurana2016cognito}
\bibfield{author}{\bibinfo{person}{Udayan Khurana}, \bibinfo{person}{Deepak
  Turaga}, \bibinfo{person}{Horst Samulowitz}, {and}
  \bibinfo{person}{Srinivasan Parthasrathy}.} \bibinfo{year}{2016}\natexlab{}.
\newblock \showarticletitle{Cognito: Automated feature engineering for
  supervised learning}. In \bibinfo{booktitle}{\emph{2016 IEEE 16th
  International Conference on Data Mining Workshops (ICDMW)}}. IEEE,
  \bibinfo{pages}{1304--1307}.
\newblock


\bibitem[\protect\citeauthoryear{Kilbertus, Rojas~Carulla, Parascandolo, Hardt,
  Janzing, and Sch{\"o}lkopf}{Kilbertus et~al\mbox{.}}{2017}]%
        {KilbertusRPHJS2017}
\bibfield{author}{\bibinfo{person}{Niki Kilbertus}, \bibinfo{person}{Mateo
  Rojas~Carulla}, \bibinfo{person}{Giambattista Parascandolo},
  \bibinfo{person}{Moritz Hardt}, \bibinfo{person}{Dominik Janzing}, {and}
  \bibinfo{person}{Bernhard Sch{\"o}lkopf}.} \bibinfo{year}{2017}\natexlab{}.
\newblock \showarticletitle{Avoiding Discrimination through Causal Reasoning}.
  In \bibinfo{booktitle}{\emph{Advances in Neural Information Processing
  Systems}}. \bibinfo{pages}{656--666}.
\newblock


\bibitem[\protect\citeauthoryear{Konda, Kumar, R{\'e}, and Sashikanth}{Konda
  et~al\mbox{.}}{2013}]%
        {konda2013feature}
\bibfield{author}{\bibinfo{person}{Pradap Konda}, \bibinfo{person}{Arun Kumar},
  \bibinfo{person}{Christopher R{\'e}}, {and} \bibinfo{person}{Vaishnavi
  Sashikanth}.} \bibinfo{year}{2013}\natexlab{}.
\newblock \showarticletitle{Feature selection in enterprise analytics: a
  demonstration using an R-based data analytics system}.
\newblock \bibinfo{journal}{\emph{PVLDB}} \bibinfo{volume}{6},
  \bibinfo{number}{12} (\bibinfo{year}{2013}), \bibinfo{pages}{1306--1309}.
\newblock


\bibitem[\protect\citeauthoryear{Kumar, Naughton, Patel, and Zhu}{Kumar
  et~al\mbox{.}}{2016}]%
        {kumar2016join}
\bibfield{author}{\bibinfo{person}{Arun Kumar}, \bibinfo{person}{Jeffrey
  Naughton}, \bibinfo{person}{Jignesh~M Patel}, {and} \bibinfo{person}{Xiaojin
  Zhu}.} \bibinfo{year}{2016}\natexlab{}.
\newblock \showarticletitle{To join or not to join? thinking twice about joins
  before feature selection}. In \bibinfo{booktitle}{\emph{Proceedings of the
  2016 International Conference on Management of Data}}.
  \bibinfo{pages}{19--34}.
\newblock


\bibitem[\protect\citeauthoryear{Kusner, Loftus, Russell, and Silva}{Kusner
  et~al\mbox{.}}{2017}]%
        {KusnerLRS2017}
\bibfield{author}{\bibinfo{person}{Matt Kusner}, \bibinfo{person}{Joshua
  Loftus}, \bibinfo{person}{Chris Russell}, {and} \bibinfo{person}{Ricardo
  Silva}.} \bibinfo{year}{2017}\natexlab{}.
\newblock \showarticletitle{Counterfactual Fairness}. In
  \bibinfo{booktitle}{\emph{Advances in Neural Information Processing
  Systems}}. \bibinfo{pages}{4066--4076}.
\newblock


\bibitem[\protect\citeauthoryear{Lauritzen and Sadeghi}{Lauritzen and
  Sadeghi}{2018}]%
        {lauritzen2018unifying}
\bibfield{author}{\bibinfo{person}{Steffen Lauritzen} {and}
  \bibinfo{person}{Kayvan Sadeghi}.} \bibinfo{year}{2018}\natexlab{}.
\newblock \showarticletitle{Unifying Markov properties for graphical models}.
\newblock \bibinfo{journal}{\emph{The Annals of Statistics}}
  \bibinfo{volume}{46}, \bibinfo{number}{5} (\bibinfo{year}{2018}),
  \bibinfo{pages}{2251--2278}.
\newblock


\bibitem[\protect\citeauthoryear{Loftus, Russell, Kusner, and Silva}{Loftus
  et~al\mbox{.}}{2018}]%
        {loftus2018causal}
\bibfield{author}{\bibinfo{person}{Joshua~R Loftus}, \bibinfo{person}{Chris
  Russell}, \bibinfo{person}{Matt~J Kusner}, {and} \bibinfo{person}{Ricardo
  Silva}.} \bibinfo{year}{2018}\natexlab{}.
\newblock \showarticletitle{Causal reasoning for algorithmic fairness}.
\newblock \bibinfo{journal}{\emph{arXiv preprint arXiv:1805.05859}}
  (\bibinfo{year}{2018}).
\newblock


\bibitem[\protect\citeauthoryear{Miller}{Miller}{2018}]%
        {miller2018open}
\bibfield{author}{\bibinfo{person}{Ren{\'e}e~J Miller}.}
  \bibinfo{year}{2018}\natexlab{}.
\newblock \showarticletitle{Open data integration}.
\newblock \bibinfo{journal}{\emph{PVLDB}} \bibinfo{volume}{11},
  \bibinfo{number}{12} (\bibinfo{year}{2018}), \bibinfo{pages}{2130--2139}.
\newblock


\bibitem[\protect\citeauthoryear{Mukherjee, Asnani, and Kannan}{Mukherjee
  et~al\mbox{.}}{2019}]%
        {mukherjee2019ccmi}
\bibfield{author}{\bibinfo{person}{Sudipto Mukherjee},
  \bibinfo{person}{Himanshu Asnani}, {and} \bibinfo{person}{Sreeram Kannan}.}
  \bibinfo{year}{2019}\natexlab{}.
\newblock \showarticletitle{Ccmi: Classifier based conditional mutual
  information estimation}.
\newblock \bibinfo{journal}{\emph{arXiv preprint arXiv:1906.01824}}
  (\bibinfo{year}{2019}).
\newblock


\bibitem[\protect\citeauthoryear{Nabi and Shpitser}{Nabi and Shpitser}{2018}]%
        {nabi2018fair}
\bibfield{author}{\bibinfo{person}{Razieh Nabi} {and} \bibinfo{person}{Ilya
  Shpitser}.} \bibinfo{year}{2018}\natexlab{}.
\newblock \showarticletitle{Fair inference on outcomes}. In
  \bibinfo{booktitle}{\emph{Thirty-Second AAAI Conference on Artificial
  Intelligence}}.
\newblock


\bibitem[\protect\citeauthoryear{Noriega~Campero, Bakker, Garcia~Bulle, and
  Pentland}{Noriega~Campero et~al\mbox{.}}{2019}]%
        {noriega2018active}
\bibfield{author}{\bibinfo{person}{Alejandro Noriega~Campero},
  \bibinfo{person}{Michiel Bakker}, \bibinfo{person}{Bernardo Garcia~Bulle},
  {and} \bibinfo{person}{Alex Pentland}.} \bibinfo{year}{2019}\natexlab{}.
\newblock \showarticletitle{Active Fairness in Algorithmic Decision Making}. In
  \bibinfo{booktitle}{\emph{Proceedings of AAAI / ACM Conference on Artificial
  Intelligence, Ethics, and Society}}. \bibinfo{pages}{77–--83}.
\newblock


\bibitem[\protect\citeauthoryear{Pearl}{Pearl}{2009}]%
        {pearl2009causality}
\bibfield{author}{\bibinfo{person}{Judea Pearl}.}
  \bibinfo{year}{2009}\natexlab{}.
\newblock \bibinfo{booktitle}{\emph{Causality}}.
\newblock \bibinfo{publisher}{Cambridge university press}.
\newblock


\bibitem[\protect\citeauthoryear{Peters, Janzing, and Sch{\"o}lkopf}{Peters
  et~al\mbox{.}}{2017}]%
        {peters2017elements}
\bibfield{author}{\bibinfo{person}{Jonas Peters}, \bibinfo{person}{Dominik
  Janzing}, {and} \bibinfo{person}{Bernhard Sch{\"o}lkopf}.}
  \bibinfo{year}{2017}\natexlab{}.
\newblock \bibinfo{booktitle}{\emph{Elements of causal inference: foundations
  and learning algorithms}}.
\newblock \bibinfo{publisher}{The MIT Press}.
\newblock


\bibitem[\protect\citeauthoryear{Russell, Kusner, Loftus, and Silva}{Russell
  et~al\mbox{.}}{2017}]%
        {russell2017worlds}
\bibfield{author}{\bibinfo{person}{Chris Russell}, \bibinfo{person}{Matt~J
  Kusner}, \bibinfo{person}{Joshua Loftus}, {and} \bibinfo{person}{Ricardo
  Silva}.} \bibinfo{year}{2017}\natexlab{}.
\newblock \showarticletitle{When worlds collide: integrating different
  counterfactual assumptions in fairness}. In
  \bibinfo{booktitle}{\emph{Advances in Neural Information Processing
  Systems}}. \bibinfo{pages}{6414--6423}.
\newblock


\bibitem[\protect\citeauthoryear{Sadeghi}{Sadeghi}{2017}]%
        {sadeghi2017faithfulness}
\bibfield{author}{\bibinfo{person}{Kayvan Sadeghi}.}
  \bibinfo{year}{2017}\natexlab{}.
\newblock \showarticletitle{Faithfulness of probability distributions and
  graphs}.
\newblock \bibinfo{journal}{\emph{Journal of Machine Learning Research}}
  \bibinfo{volume}{18}, \bibinfo{number}{148} (\bibinfo{year}{2017}),
  \bibinfo{pages}{1--29}.
\newblock


\bibitem[\protect\citeauthoryear{Salimi, Rodriguez, Howe, and Suciu}{Salimi
  et~al\mbox{.}}{2019}]%
        {capuchin}
\bibfield{author}{\bibinfo{person}{Babak Salimi}, \bibinfo{person}{Luke
  Rodriguez}, \bibinfo{person}{Bill Howe}, {and} \bibinfo{person}{Dan Suciu}.}
  \bibinfo{year}{2019}\natexlab{}.
\newblock \showarticletitle{Interventional fairness: Causal database repair for
  algorithmic fairness}. In \bibinfo{booktitle}{\emph{Proceedings of the 2019
  International Conference on Management of Data}}. \bibinfo{pages}{793--810}.
\newblock


\bibitem[\protect\citeauthoryear{Schelter, He, Khilnani, and
  Stoyanovich}{Schelter et~al\mbox{.}}{2019}]%
        {SchelterHKS2019}
\bibfield{author}{\bibinfo{person}{Sebastian Schelter}, \bibinfo{person}{Yuxuan
  He}, \bibinfo{person}{Jatin Khilnani}, {and} \bibinfo{person}{Julia
  Stoyanovich}.} \bibinfo{year}{2019}\natexlab{}.
\newblock \bibinfo{title}{{FairPrep}: Promoting Data to a First-Class Citizen
  in Studies on Fairness-Enhancing Interventions}.
\newblock \bibinfo{howpublished}{arXiv:1911.12587}.
\newblock


\bibitem[\protect\citeauthoryear{Singh, Singh, Mhasawade, and Chunara}{Singh
  et~al\mbox{.}}{2019}]%
        {singh2019fair}
\bibfield{author}{\bibinfo{person}{Harvineet Singh}, \bibinfo{person}{Rina
  Singh}, \bibinfo{person}{Vishwali Mhasawade}, {and} \bibinfo{person}{Rumi
  Chunara}.} \bibinfo{year}{2019}\natexlab{}.
\newblock \showarticletitle{Fair predictors under distribution shift}.
\newblock \bibinfo{journal}{\emph{arXiv preprint arXiv:1911.00677}}
  (\bibinfo{year}{2019}).
\newblock


\bibitem[\protect\citeauthoryear{Spirtes, Glymour, Scheines, and
  Heckerman}{Spirtes et~al\mbox{.}}{2000}]%
        {spirtes2000causation}
\bibfield{author}{\bibinfo{person}{Peter Spirtes}, \bibinfo{person}{Clark~N
  Glymour}, \bibinfo{person}{Richard Scheines}, {and} \bibinfo{person}{David
  Heckerman}.} \bibinfo{year}{2000}\natexlab{}.
\newblock \bibinfo{booktitle}{\emph{Causation, prediction, and search}}.
\newblock \bibinfo{publisher}{MIT press}.
\newblock


\bibitem[\protect\citeauthoryear{Strobl, Zhang, and Visweswaran}{Strobl
  et~al\mbox{.}}{2019}]%
        {strobl2019approximate}
\bibfield{author}{\bibinfo{person}{Eric~V Strobl}, \bibinfo{person}{Kun Zhang},
  {and} \bibinfo{person}{Shyam Visweswaran}.} \bibinfo{year}{2019}\natexlab{}.
\newblock \showarticletitle{Approximate kernel-based conditional independence
  tests for fast non-parametric causal discovery}.
\newblock \bibinfo{journal}{\emph{Journal of Causal Inference}}
  \bibinfo{volume}{7}, \bibinfo{number}{1} (\bibinfo{year}{2019}).
\newblock


\bibitem[\protect\citeauthoryear{Xu, Wu, Yuan, Zhang, and Wu}{Xu
  et~al\mbox{.}}{2019}]%
        {xu2019achieving}
\bibfield{author}{\bibinfo{person}{Depeng Xu}, \bibinfo{person}{Yongkai Wu},
  \bibinfo{person}{Shuhan Yuan}, \bibinfo{person}{Lu Zhang}, {and}
  \bibinfo{person}{Xintao Wu}.} \bibinfo{year}{2019}\natexlab{}.
\newblock \showarticletitle{Achieving causal fairness through generative
  adversarial networks}.
\newblock \bibinfo{journal}{\emph{IJCAI}}.
\newblock


\bibitem[\protect\citeauthoryear{Zafar, Valera, Rogriguez, and Gummadi}{Zafar
  et~al\mbox{.}}{2017}]%
        {zafar2017fairness}
\bibfield{author}{\bibinfo{person}{Muhammad~Bilal Zafar},
  \bibinfo{person}{Isabel Valera}, \bibinfo{person}{Manuel~Gomez Rogriguez},
  {and} \bibinfo{person}{Krishna~P Gummadi}.} \bibinfo{year}{2017}\natexlab{}.
\newblock \showarticletitle{Fairness constraints: Mechanisms for fair
  classification}. In \bibinfo{booktitle}{\emph{Artificial Intelligence and
  Statistics}}. PMLR, \bibinfo{pages}{962--970}.
\newblock


\bibitem[\protect\citeauthoryear{Zhang, Kumar, and R{\'e}}{Zhang
  et~al\mbox{.}}{2016a}]%
        {zhang2016materialization}
\bibfield{author}{\bibinfo{person}{Ce Zhang}, \bibinfo{person}{Arun Kumar},
  {and} \bibinfo{person}{Christopher R{\'e}}.}
  \bibinfo{year}{2016}\natexlab{a}.
\newblock \showarticletitle{Materialization optimizations for feature selection
  workloads}.
\newblock \bibinfo{journal}{\emph{ACM Transactions on Database Systems (TODS)}}
  \bibinfo{volume}{41}, \bibinfo{number}{1} (\bibinfo{year}{2016}),
  \bibinfo{pages}{2}.
\newblock


\bibitem[\protect\citeauthoryear{Zhang and Bareinboim}{Zhang and
  Bareinboim}{2018a}]%
        {ZhangB2018b}
\bibfield{author}{\bibinfo{person}{Junzhe Zhang} {and} \bibinfo{person}{Elias
  Bareinboim}.} \bibinfo{year}{2018}\natexlab{a}.
\newblock \showarticletitle{Equality of Opportunity in Classification: A Causal
  Approach}. In \bibinfo{booktitle}{\emph{Advances in Neural Information
  Processing Systems}}. \bibinfo{pages}{3671--3681}.
\newblock


\bibitem[\protect\citeauthoryear{Zhang and Bareinboim}{Zhang and
  Bareinboim}{2018b}]%
        {ZhangB2018a}
\bibfield{author}{\bibinfo{person}{Junzhe Zhang} {and} \bibinfo{person}{Elias
  Bareinboim}.} \bibinfo{year}{2018}\natexlab{b}.
\newblock \showarticletitle{Fairness in Decision-Making --- The Causal
  Explanation Formula}. In \bibinfo{booktitle}{\emph{Proceedings of the AAAI
  Conference on Artificial Intelligence}}. \bibinfo{pages}{2037--2045}.
\newblock


\bibitem[\protect\citeauthoryear{Zhang, Wu, and Wu}{Zhang
  et~al\mbox{.}}{2016b}]%
        {zhang2016causal}
\bibfield{author}{\bibinfo{person}{Lu Zhang}, \bibinfo{person}{Yongkai Wu},
  {and} \bibinfo{person}{Xintao Wu}.} \bibinfo{year}{2016}\natexlab{b}.
\newblock \showarticletitle{A causal framework for discovering and removing
  direct and indirect discrimination}.
\newblock \bibinfo{journal}{\emph{arXiv preprint arXiv:1611.07509}}
  (\bibinfo{year}{2016}).
\newblock


\end{thebibliography}
\pagebreak





\section{ Proofs}
First, we show the following property of \texttt{do-calculus}.
\begin{lemma}
Given a disjoint collection of variables $X$, $Y$ and $Z$ in a causal graph $G$, such that $(X\bigCI Y | Z')$, where $Z'\subseteq Z$, then $Pr[X|do(Y),do(Z)] = Pr[X|do(Z)]$\label{lem:thirdrule}
\end{lemma}
\begin{proof}
Using the third rule of \texttt{do-calculus} (Equation 10, \citep{huang2012pearl}),
$Pr[X|do(Y),do(Z)] = Pr[X|do(Z)]$ when $X$ is independent of $Y$ given $Z$ in the graph where incoming edges of $Z$ have been removed. Since, $X\bigCI Y | Z'$ in $G$ where $Z'\subseteq Z$, removing additional incoming edges will ensure that none of the variables in $Z$ are a collider and conditioning on $Z\setminus Z'$ additionally will still maintain conditional independence.
\end{proof}


\begin{lemma}
Given a dataset $D$ comprising of variables $\mathcal{A}\cup\mathcal{S}\cup\mathcal{X}$, target variable $Y$ and let $Y'$ be the variable learnt using the feature subset \sainyam{$\mathcal{T}\cup \mathcal{A}$},  then  $Pr(Y'|do(\mathcal{A}),do(\mathcal{S}),\mathcal{T}) = Pr(Y'|do(\mathcal{A}),\mathcal{T})$, where $T\subseteq \mathcal{X}$\label{lem:yprime}
\end{lemma}
\begin{proof}
Based on the assumption about the construction of $Y'$ (Assumption \ref{yprime-assumption}), the variable $Y'$ is only dependent on the variables in $\mathcal{A}\cup \mathcal{T}$ in all environments. Given $\mathcal{A}\cup \mathcal{T}$, the variable $Y'$ is independent of $\mathcal{S}$. The same condition holds even when incoming edges of $\mathcal{A}$ are removed. Also, $\mathcal{S}$ nodes do not have any incoming edges. Therefore, on applying the third rule of \texttt{do-calculus}, since $Y'$ is independent of $\mathcal{S}$ in the modified graph where incoming edges of $\mathcal{A}$ and $\mathcal{S}$ nodes that are ancestors of $\mathcal{T}$ are removed. 
Therefore,
$Pr(Y'|do(\mathcal{A}),do(\mathcal{S}),\mathcal{T}) = Pr(Y'|do(\mathcal{A}),\mathcal{T})$
\end{proof}

\subsection{Proof of Lemma~$7$}
We denote conditional mutual information between two variables $X$ and $Y$ given $Z$ as $I(X,Y|Z)$.
\begin{proof}
Using chain rule, $I(X_1, \mathcal{X}|Z) = I(X_1,X_i|Z) + I(X_1,Z|X_i) \ge I(X_1,X_i|Z) >  0$
\end{proof}

\subsection{Proof of Lemma~$8$}
\begin{proof}
$X_1\nbigCI \mathcal{X}\setminus X_1|Z$ means that path from $X_1$ to $\mathcal{X}$ is not blocked. Using  assumption \ref{faithfulness}, that the path to atleast one of $X_i\in \mathcal{X}\setminus X_1$ is not blocked. Hence, $\exists i$ such that $X_1\nbigCI X_i|Z$.
\end{proof}
\begin{figure}
\includegraphics[width=0.35\textwidth]{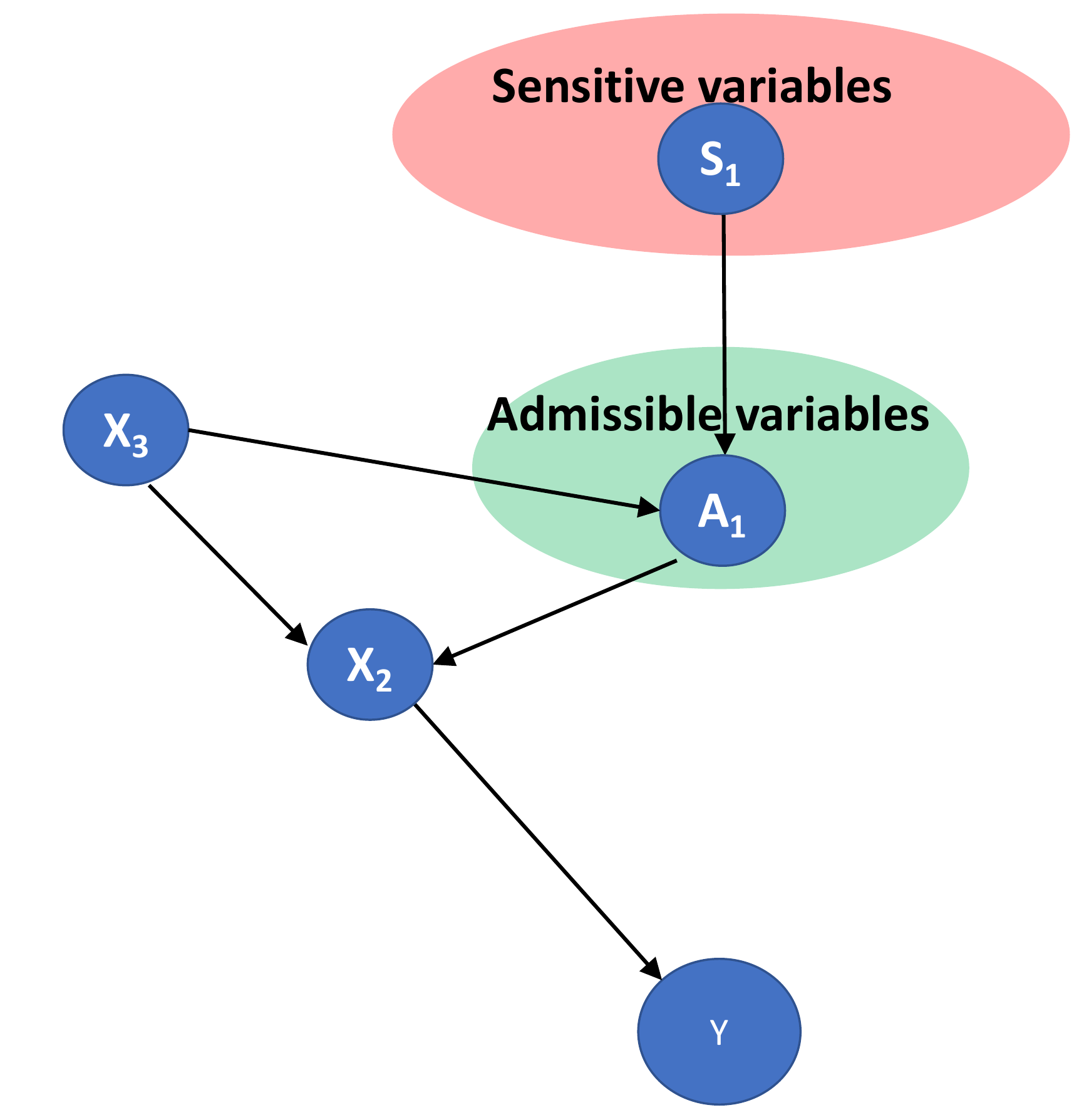}
\caption{Example graph where $X_2$ is not identified as causally fair by \texttt{GrpSel}. We omit other nodes for the sake of clarity.\label{fig:hard}}
\end{figure}

\subsection{Dataset description}
\begin{itemize}
\item \textit{Medical Expenditure} (MEPS)\footnote{\url{https://meps.ahrq.gov/mepsweb/}}: This dataset comprises of health assessment features (both physical and mental) along with demographic features. The dataset is used to predict the number of hospital visits. 
\item \textit{German Credit}\footnote{\url{https://archive.ics.uci.edu/ml/datasets/statlog+(german+credit+data)}} dataset from UCI repository contains attributes of various applicants and the goal is to classify them based on credit risk.
\item \textit{Compas}\footnote{\url{https://github.com/propublica/compas-analysis}} was a risk assessment tool used by courts to determine if a  defendant should be released or retained. This dataset contains features like age, race, prior conviction, etc.
\end{itemize}
In addition to the default set of features, we use techniques from \citep{khurana2016cognito} to generate  new features, constructed by composition of already present features. 

\noindent \textbf{Setup.} We considered the default threshold of p-value to be 0.01 and  default settings of sklearn's logistic regression classifier.  \texttt{GrpSel} and \texttt{SeqSel} were implemented in R and the classifier training and testing in Python. The code was run on a laptop with 16GB RAM running MAC OS. 

\section{Additional Experiments}
Our experiments on real-world datasets that compare group fairness metric (absolute odds difference) and conditional mutual information (CMI) correspond two ends of the spectrum.  Since causal fairness implies group fairness, Figure~\ref{fig:quality1} provides some evidence that our algorithms can potentially ensure fairness.
On the other hand, since \texttt{GrpSel} has low CMI with the target variable given $\mathcal{A}$ (Table~\ref{tab:cmi}), the CMI of $\mathcal{S}$ and $Y'$ will be low even after intervening on $\mathcal{A}$. This experiment guarantees the effectiveness of our techniques to ensure causal fairness.

To further analyze the ability of our algorithms to ensure causal fairness, we evaluate \texttt{GrpSel} and \texttt{SeqSel} on multiple synthetic datasets generated using causal graphs of varied sizes (1000, 3000 and 5000) along with the examples shown in Figure 1 a-c.

In this experiment, we validated the effectiveness of \texttt{SeqSel} and \texttt{GrpSel} to identify the variables that ensure causal fairness.
Across all datasets, we observed that \texttt{SeqSel} and \texttt{GrpSel} identified all the variables that ensure causal fairness. One of the variables in 1000 node dataset was not detected by our algorithm. We show a small subgraph of this dataset in Figure~\ref{fig:hard}. In this dataset, variable $X_2$ is not identified by \texttt{GrpSel} and \texttt{SeqSel} because $X_2\nbigCI S_1$ and $X_2\nbigCI S_1 | A_1$.  This is an example scenario where interventional data is required to identify such variables.

We ran an additional experiment to test the robustness of our techniques with respect to distribution shift. In this experiment, we varied the effect of sensitive attribute on the target variable through specific attributes. This shift in data distribution did not affect the performance of \texttt{GrpSel} or \texttt{SeqSel} but pre-processing techniques like reweighting\footnote{\url{https://aif360.mybluemix.net/}} fail to ensure fairnss under the modified distribution.





\subsection{d-separation}
Two nodes X and Y are d-separated if every path between them (should any exist) is blocked.
If even one path between X and Y is unblocked, X and Y are d-connected. 
More formally, 
\begin{definition}[d-separation]
A path $p$ is blocked by a set of nodes $Z$ if and only if
\begin{enumerate}
    \item p contains a chain of nodes $A \rightarrow B \rightarrow C$ or a fork $A \leftarrow B \rightarrow C$ such that the middle node B
is in Z (i.e., B is conditioned on), or
    \item  p contains a collider $A \rightarrow B \leftarrow C$ such that the collision node B is not in Z, and no descendant of B is in Z.
\end{enumerate}
If Z blocks every path between two nodes X and Y, then X and Y are d-separated, conditional
on Z, and thus are independent conditional on Z.
\end{definition}
\end{document}